%% file: main.tex
\documentclass{article}

\usepackage[utf8]{inputenc}
\usepackage{hyperref}
\usepackage[square, numbers]{natbib}
\usepackage[english]{babel}
\usepackage{booktabs} % For formal tables
\usepackage[ruled]{algorithm2e} % For algorithms

\usepackage{fullpage}
\usepackage{url}
\usepackage{graphicx}
\usepackage{amsmath,amssymb,dsfont,xcolor,amssymb}
\usepackage{amsthm}
\usepackage{bm}
\usepackage{mathrsfs}
\usepackage{appendix}
\usepackage{graphicx}
\usepackage{pifont}
\usepackage{xcolor}
\usepackage{xcolor}
\usepackage{algpseudocode}
\usepackage{bbm}
\usepackage{thm-restate}
\newtheorem{thm}{Theorem}
\newtheorem{lemma}{Lemma}
\newtheorem{definition}{Definition}
\newtheorem{cor}{Corollary}

\newtheorem{fact}{Fact}
\newtheorem{clm}{Claim}

\DeclareMathOperator*{\argmax}{argmax}
\DeclareMathOperator*{\argmin}{argmin}
\newcommand{\Hs}{\mathcal{H}}
\newcommand{\X}{\mathcal{X}}
\newcommand{\Y}{\mathcal{Y}}
\newcommand{\Z}{\mathcal{Z}}
\newcommand{\R}{\mathbb{R}}
\newcommand{\E}{\mathbb{E}}
\newcommand{\F}{\mathcal{F}}
\newcommand{\epsvec}{\vec{\epsilon}}
\newcommand{\opt}{\text{OPT}}
\newcommand{\1}{\mathds{1}}
\newcommand{\Ps}{\mathcal{P}}

\title{Lexicographically Fair Learning:\\ Algorithms and Generalization\footnote{Author emails: ediana@wharton.upenn.edu, wesgill@seas.upenn.edu, igh@seas.upenn.edu, mkearns@cis.upenn.edu, aaroth@cis.upenn.edu, saeedsh@wharton.upenn.edu.}}
\author{Emily Diana \; \and Wesley Gill \; \and Ira Globus-Harris \; \and Michael Kearns \; \and Aaron Roth \; \and Saeed Sharifi-Malvajerdi}

\date{\textit{University of Pennsylvania}}

\begin{document}

\maketitle
\begin{abstract}
\input{abstract}
\end{abstract}

 \thispagestyle{empty} \setcounter{page}{0}
 \clearpage

\section{Introduction}
\label{sec:intro}

Most notions of statistical group fairness ask that a model approximately equalize some error statistic across demographic groups. Often this is motivated as a tradeoff: the goal is to lower the error of the most disadvantaged group, and if doing so requires increasing the error on some more advantaged group, so be it---this is a cost that we are willing to pay in the name of equity. But solutions which equalize group errors do \emph{not} in general mediate a clean tradeoff in which losses in accuracy on more advantaged groups result in increases in accuracy on less advantaged groups: instead, generically (i.e. except in the very special case in which the Bayes optimal error is identical for all groups), a constraint of equalizing group error rates may  require \emph{artificially increasing} the error on at least one group, without any corresponding benefit to any other group. 

A partial answer to this criticism of standard notions of group fairness is the classical notion of \emph{minimax fairness}, recently studied by \cite{minimax1,minimax2} in the context of supervised learning. Minimax fairness asks for a model which minimizes the error of the group \emph{most disadvantaged} by the model---i.e. the group with maximum group error. In doing so, it realizes the promise of equal error solutions in that it trades off higher error on populations more advantaged by the model for lower error on populations less advantaged by the model when this is possible---but without artificially increasing the error of any group when doing so.  Indeed, it is not hard to see that a minimax model necessarily \emph{weakly Pareto dominates} an equal error rate model, in the sense that group errors are only lower in the minimax solution \emph{simultaneously for all groups}. 

This narrative is most sensible if there are only two demographic groups of interest. If there are more than two groups, there may be many different minimax optimal models that have very different error profiles for groups other than the max error group. How should we choose amongst these? Prior work \cite{minimax2} has broken ties by optimizing for overall classification accuracy. But why should we entirely give up on the goal of optimizing for the most disadvantaged, partially enunciated in the motivation of minimax fairness, once we have fixed the error of only one of many groups? 

In this paper we propose the natural continuation of this idea, which we call \emph{lexicographic minimax fairness}. Informally speaking, this notion recurses on the idea that we wish to minimize the cost of the least well off. A model that satisfies lexicographic fairness, which we call a \emph{lexifair} model, will minimize the maximum error $\gamma_1$  on any group, amongst all possible models (i.e. a lexifair model is a also a minimax model). Further, amongst the set of all minimax models, a lexifair model must minimize the error of the group with the second highest error $\gamma_2$. Amongst all of these models, it further minimizes the error of the group with the third highest error $\gamma_3$, and so on.\footnote{It is easy to see that there are cases in which a lexifair model may have arbitrarily smaller errors than a minimax model on all but the worst-off group.}

\subsection{Our Contributions}
Our first contribution is a definition of (approximate) lexicographic minimax fairness. Correctly defining an actionable notion of lexicographic minimax fairness is surprisingly subtle. For standard computational and statistical reasons, it will not be possible to exactly match the distributional lexicographically optimal error rates $\gamma_1,\gamma_2,\gamma_3,$ etc. But as we will observe, these lexicographically optimal error rates can be arbitrarily unstable, in the sense that amongst the set of models that have minimax error larger than $\gamma_1$ by even an arbitrarily small margin, the value of the optimal lexifair error on the third highest error group $\gamma_3'$ can be arbitrarily larger than $\gamma_3$ (See our example in Section \ref{sec:stability}). An implication of this is that the vectors of errors $\gamma$, $\gamma'$ representing exact lexifair solutions in and out of sample can be entirely incomparable and arbitrarily different from one another. Hence we need a definition of approximate lexifairness that accounts for this instability, and allows for sensible statements about approximation and generalization. 

Another challenge arises in the interaction between our definitions and our (desired) algorithms.  A constraint on the \emph{highest} error amongst all groups, which arises in defining minimax error, is convex, and hence amenable to algorithmic optimization. However, naive specifications of lexifairness involve constraints on the second highest group errors, the third highest group errors, and more generally $k$th highest errors. These are non-convex constraints when taken in isolation. However, as it turns out, a constraint on the second highest error becomes convex when we restrict attention to minimax optimal classifiers, and more generally, a constraint on the $k$th highest error becomes convex once the values of the lower order group errors are constrained to their lexifair values. We show this by giving a clearly convex variant of our lexifair definition, specified by exponentially many \emph{linear constraints}, which  replace constraints on the $k$'th highest error groups  with constraints on the \emph{sums} of all $k$-tuples of group errors. We then show that our definition of ``convex lexifairness'' is equivalent to our original notion of lexifairness, at least in the exact case (absent approximation). 
We give our formal definitions in Section \ref{sec:gooddef}. 

With our notion of approximate lexifairness in hand and our convexified constraints, we give \emph{oracle-efficient} algorithms for finding approximate lexifair models in both the regression and classification case. This means that our algorithms are efficient reductions to the problem of unconstrained (that is, standard non-fair) learning over the same model class.  Despite the worst-case intractability of most natural learning problems even absent fairness considerations, a desirable feature of oracle-efficient algorithms is that they can be implemented using any of the common and practical heursitics for non-fair learning, often with good empirical success~\cite{kearns2019empirical,kearns2019average, jung2019elicitation, agarwal2018reductions}.

Our algorithms are based on solving the corresponding constrained optimization problem by recasting it  as a (Lagrangian) minmax optimization problem, and using no-regret dynamics. Because our ``convexified'' lexifairness constraints are exponentially numerous, the ``constraint player'' in our formulation has exponentially many strategies --- but as we show, we can efficiently optimize over her strategy space using an efficient separation oracle. Hence the constraint player can always play according to a ``best response'' strategy in our simulated dynamics.  When our base model class is continuous and our loss function convex (as it is with e.g. linear regression), then the ``learner'' in our dynamics can play gradient descent over parameter space. In this case, our oracle efficient-algorithms are in fact fully polynomial time algorithms because our reduction to weighted learning problems involves only \emph{non-negative} weights, which preserves convexity. In the classification case, when our loss function is \emph{non-convex}, we can  convexify it by considering the set of all probability distributions over base models. Here the parameters we optimize over become the weights of the probability distribution, and our loss function (i.e. the expected loss over the choice of a random model) becomes linear in our (enormous) parameter space. In this case, we are effectively solving a linear program that has both exponentially many variables and exponentially many constraints --- but we are nevertheless able to do so in an oracle-efficient manner by making appropriate use of the Follow the Perturbed Leader algorithm \cite{KALAI2005291} for no-regret learning. 

Finally, we prove a generalization theorem, showing that if we have a dataset $S$  (sampled i.i.d. from an underlying distribution) that has sufficiently many samples from each group, and if we have a model that is approximately lexifair for $S$, then the model is also approximately lexifair on the underlying distribution. This is significantly more involved than just a standard uniform convergence argument --- which would simply state that our in and out of sample errors on each group are close to one another --- because approximate lexifairness additionally depends on the precise \emph{relationship} between these group errors. Nevertheless, we show that uniform convergence is a sufficient condition to guarantee that in-sample lexifairness bounds correspond to out of sample lexifairness bounds. 

\subsection{Related Work}
\input{relatedWork}

\section{Model and Definitions}
\label{sec:model}
Let $\Z = \X \times \Y$ be an arbitrary data domain. Each data point in our setting is a pair $z = (x,y)$ where $x \in \X$ is the feature vector and $y \in \Y$ is the response variable (i.e. the label). Let $\X$ consist of points belonging to $K$ (not necessarily disjoint) groups $\mathcal{G}_1, \ldots, \mathcal{G}_K$, so we can write $\X = \cup_{k=1}^K \mathcal{G}_k.$ We write $\Ps$ to denote an arbitrary distribution over $\Z$,  and $\Ps_k$ to denote the marginal distribution induced by $\Ps$ on the $k$th group $\mathcal{G}_k \times \Y$. Let $S = \{z_i\}_{i=1}^n$ be a data set of size $n$, which for the purposes of proving generalization bounds, we will take to consist of $n$ data points drawn i.i.d. from $\Ps$.  Denote the points in S that are contained in $\mathcal{G}_k$ by $G_k$, so we can write $S = \cup_{k=1}^K G_k$.

Let $\Hs \subseteq \left\{ h: \X \to \Y \right\}$ be the model class of interest, and let $L: \Hs \times \Z \to \R_{+}$ be a loss function that takes a data point $z$ and a model $h$ as inputs, and outputs the loss of $h$ on $z$. For instance, in the case of classification and zero-one loss, we have $L(h,z) = \1 \left[ h(x) \neq y \right]$. We will abuse notation and write $L_z (\cdot)$ for $L(\cdot, z)$ for any data point $z$. Throughout the paper, for any distribution $\Ps$, we write the expected loss of a model $h$ over $\Ps$ as:
\[
L_{\Ps} (h) \triangleq L(h, \Ps) \triangleq \E_{z \sim \Ps} \left[ L_z ( h )\right]
\]
We slightly abuse notation and write $L_S(h)$ to denote the empirical loss on a dataset $S$. Here and throughout the paper when $S$ plays the role of a distribution, we interpret that as the \emph{uniform distribution} over the points in $S$, and accordingly, $z \sim S$ as a point sampled \emph{uniformly at random} from $S$. 

Until Section \ref{sec:generalization}, we will work exclusively with sample quantities, and so for simplicity of notation, let us define $L_k (h) \triangleq L_{G_k} (h)$ to denote the \emph{sample} loss of a model $h$ on the $k$'th group. When necessary, we will write $L_k \left(h, \Ps \right)$ to denote $L_{\Ps_k} \left(h \right)$, the corresponding \emph{distributional} loss of $h$ on the $k$'th group. For any model $h$ and any data set $S = \cup_k \{ G_k \}$, let $\bar{h}_S$ be the ordering induced on the groups $\{ G_k \}_{k=1}^K$ by the loss of $h$, breaking ties arbitrarily. In other words, $\bar{h}_S: [K] \to [K]$ is any bijection such that the following condition holds:
$
L_{\bar{h}_S(1)} (h) \ge L_{\bar{h}_S(2)} (h) \ge \ldots \ge L_{\bar{h}_S(K)} (h)
$. The corresponding distributional ordering of the groups by any model $h$ is defined similarly: for any model $h$ and any distribution $\Ps$ over $\Z$, let $\bar{h}_\Ps: [K] \to [K]$ be the ordering induced on the groups $\{ \mathcal{G}_k \}_{k=1}^K$ by the expected loss of $h$, breaking ties arbitrarily. In other words, $\bar{h}_\Ps$ is any bijection such that the following condition holds:
$
L_{\bar{h}_\Ps (1)} (h, \Ps) \ge L_{\bar{h}_\Ps (2)} (h, \Ps) \ge \ldots \ge L_{\bar{h}_\Ps (K)} (h, \Ps)
$. When the distribution (data set) is clear from context, we elide the dependence on the distribution (data set) and simply write $\bar{h}$ for $\bar{h}_\Ps$ ($\bar{h}_S$). 

Our definition of lexifairness will be given recursively. At the base level, we  define $\Hs_{(0)} = \Hs$ to be the set of all models in our class. Then recursively for all $1 \le j \le K$, we define:
\[
\gamma_j \triangleq \min_{h \in \Hs_{(j-1)}} L_{\bar{h}(j)}(h), \quad \Hs_{(j)} \triangleq \left\{ h \in \Hs_{(j-1)}: L_{\bar{h}(j)}(h) = \gamma_j \right\}
\]

In words, $\gamma_j$ is the smallest error that any model in $\Hs_{(j-1)}$ obtains on the group that has the $j$th highest error, and $\Hs_{(j)}$ is the set of \emph{all} models in $\Hs_{(j-1)}$ that attain this minimum --- i.e. that have $j$th highest error equal to $\gamma_j$. Thus, $\gamma_1$ is the minimax error --- i.e. the highest group error for the model that is chosen to \emph{minimize} the maximum group error. Similarly, $\gamma_2$ is the error of the second highest group for all minimax optimal models that further minimize the error of the second highest group, and so on. With this notation in hand, we can define exact lexifairness as follows:

\begin{definition}[Exact Lexicographic Fairness]\label{def:exactlexifair}
    Let $1 \le \ell \le K$. We say a model $h \in \Hs$ satisfies level-$\ell$ (exact) lexicographic fairness (lexifairness) if for all $j \le \ell$, $L_{\bar{h}(j)} (h) \le \gamma_j$.
    %If $l = K$, we say that the model satisfies (exact) Lexifairness. 
\end{definition}

Minimax fairness corresponds to level-1 lexifairness. This is a definition of \emph{exact} lexifairness, in that it permits no approximation to the error rates --- i.e. we require $L_{\bar{h}(j)} (h) \le \gamma_j$ for all $j$, and hence $L_{\bar{h}(j)} (h) = \gamma_j$ for all $j$. For a variety of reasons, we will need definitions that tolerate approximation. For example, because we inevitably have to train on a fixed dataset, but want our guarantees to generalize to new datasets drawn from the same distribution, we will need to accommodate statistical approximation. The optimization techniques we will bring to bear will also only be able to \emph{approximate} lexifairness, even in sample. But it turns out that defining a sensible approximate notion of lexifairness is more subtle than it first appears.

\subsection{Approximate Lexifairness: Stability and Convexity}

We begin with the ``obvious'' but ultimately flawed definition of approximate lexifairness (Definition~\ref{def:failedlexifair}), and then explain why it is lacking in stability. This will lead us to the definitions we finally adopt: Definition \ref{def:lexifair} and its \emph{convexified} version (Definition~\ref{def:convexlexifair}), 
which we show is equivalent (Claim \ref{clm:relationship}), and
for which we can develop efficient algorithms.

\subsubsection{The Challenge of Stability}
\label{sec:stability}
The most natural seeming definition of approximate lexifairness begins with our notion of exact lexifairness (Definition \ref{def:exactlexifair}), and adds slack to all of the inequalities contained within. In other words, we attempt to find a model that has sorted group errors $\gamma_1',\gamma_2',\ldots,\gamma_K'$ that pointwise approximate the optimal lexifair vector of sorted group errors $\gamma_1,\ldots,\gamma_K$. 

\begin{definition}[A Flawed Definition]\label{def:failedlexifair}
    Let $1 \le \ell \le K$ and $\alpha \ge 0$. We say a model $h \in \Hs$ satisfies $(\ell, \alpha)$-lexicographic fairness if for all $j \le \ell$, $L_{\bar{h}(j)} (h) \le \gamma_j + \alpha$.
\end{definition}

To see the problem with the above definition, consider a setting with three groups, and a model class $\Hs$ that contains all distributions (or randomized classifiers) over two pure classifiers $ \{h_1,h_2\}$. Imagine that $h_1$ induces the (unsorted) vector of group error rates $\langle 0.5, 0.5, 0 \rangle$, and $h_2$ induces the  (unsorted) vector of  group error rates $\langle 0.5 + 2\alpha, 0, 0.5 \rangle$, for some arbitrarily small $\alpha > 0$. Note that it is easy to construct distributions over labeled instances with exactly these group error vectors by simply arranging each classifier to disagree with the labels on the specified fraction of a group. So, for simplicity we abstract away the data and directly discuss the error vectors. 

The minimax group error for this model class is $\gamma_1 = 0.5$, and is achieved only by $h_1$ which has error 0.5 on the first and second groups. Since the largest group error of $h_2$ is also on the first group with value $0.5 + 2\alpha > 0.5$, any distribution over $\{h_1,h_2\}$ that places a non-zero probability on $h_2$ will therefore violate the (exact) minimax constraint. This in turn implies that $\Hs_{(1)} = \{ h_1 \}$. Therefore, the only exact lexifair model is $h_1$ and thus $\gamma_1 = 0.5$, $\gamma_2 = 0.5$, $\gamma_3 = 0$.

However, imagine that because of estimation error (as is inevitable if we are learning based on a finite sample) or optimization error (since we generally don't have access to exact optimization oracles in learning settings), we slightly misestimate the minimax group error $\gamma_1$ to be $\gamma_1' = 0.5 + \alpha$. If we now optimize, allowing the largest group error to be as much as $\gamma_1' = 0.5 +\alpha$, we may now find randomized classifiers which put  weight as large as 0.5 on $h_2$. The uniform distribution over $\{ h_1, h_2\}$ induces the unsorted vector of group errors  $\langle 0.5 + \alpha, 0.25, 0.25 \rangle$. The induced error on the second group (which is now also the group with second largest error) of 0.25 is considerably \emph{smaller} than $\gamma_2 = 0.5$. So far this appears to be all right, since $\gamma'_2 < \gamma_2$. But if we now attempt to optimize the error of the third highest error $\gamma'_3$, \emph{subject to the constraint} that the largest group error is (close to) $\gamma'_1$ and the second largest group error is (close to) $\gamma_2'$, we now find that we are forced  to settle for third highest group error $\gamma'_3 \approx 0.25$, which is considerably \emph{larger} than the value of the third highest group's error of $\gamma_3 = 0$ in the exact lexifair solution.

This example highlights a fundamental \emph{instability} of our first (flawed) attempt at defining approximate lexifairness: even arbitrarily small estimation (or optimization) error introduced to the minimax error rate $\gamma_1$ can result in large, non-monotonic effects for later group errors --- enforcing even a valid \emph{upper bound} on $\gamma_1$ can cause $\gamma_3$ to increase substantially, and these effects compound even further if we have more than three groups. 

\subsubsection{A Stable and Convex Definition}
\label{sec:gooddef}
With the proceeding example of the instability inherent in our (flawed) Definition \ref{def:failedlexifair}, we now give the definition of approximate lexifairness that we begin with:
\begin{definition}[Approximate Lexicographic Fairness]\label{def:lexifair}
    Fix a distribution $\Ps$. Let $1 \le \ell \le K$ and $\alpha \ge 0$. For any sequence of mappings $\epsvec = \left( \epsilon_1, \epsilon_2, \ldots, \epsilon_\ell \right)$ where $\epsilon_j \in \R^{\Hs}$, define $\Hs^{\epsvec}_{(0)} (\Ps) \triangleq \Hs$, and recursively for all $1 \le j \le \ell$ define:
    \[
    \Hs^{\epsvec}_{(j)} (\Ps) \triangleq \left\{h \in \Hs^{\epsvec}_{(j-1)} (\Ps): L_{\bar{h}(j)} (h, \Ps) \le \min_{g \in \Hs^{\epsvec}_{(j-1)} (\Ps)} L_{\bar{g}(j)} (g, \Ps) + \epsilon_j (h) \right\}
    \]
    and let $\Vert \epsvec \Vert_\infty = \max_{1 \le j \le \ell} \max_{h \in \Hs} \epsilon_j (h)$. We say a model $h \in \Hs$ satisfies $(\ell, \alpha)$-lexicographic fairness (``lexifairness'') with respect to $\Ps$ if there exists $\epsvec$ with $\Vert \epsvec \Vert_\infty \le \alpha$ such that for all $j \le \ell$:
    \[
    L_{\bar{h}(j)} (h, \Ps) \le \min_{g \in \Hs^{\epsvec}_{(j-1)} (\Ps)} L_{\bar{g}(j)} (g, \Ps) + \epsilon_j (h) + \alpha
    \]
    When we prove bounds on empirical lexifairness, we simply take the distribution to be the uniform distribution over the data set $S$. When the distribution is clear from context, we will write $ \Hs_{(j)}^{\epsvec}$ and elide the dependence on the distribution.
\end{definition}

Note that there are two distinctions between Definition \ref{def:lexifair} and Definition \ref{def:failedlexifair}. First, the recursively defined sets $\Hs^{\epsvec}_{(j)}$ now incorporate some $\epsilon_j (\cdot)$ slack in their parameterization which will help capture statistical (or optimization) error. Second (and crucially), we now call a solution $(\ell,\alpha)$-approximately lexifair if it satisfies our requirements for \emph{some} sequence of relaxations $\epsvec$ that is component-wise less than $\alpha$ for all models $h$. It is this second point that avoids the instability and non-monotonicity that arises from Definition \ref{def:failedlexifair}. We observe that Definition \ref{def:lexifair} is a strict weakening of Definition \ref{def:failedlexifair}:

\begin{clm}
Definition~\ref{def:lexifair} is a relaxation of Definition~\ref{def:failedlexifair}: if a model satisfies $(\ell,\alpha)$-lexicographic fairness according to Definition~\ref{def:failedlexifair}, then it also satisfies $(\ell,\alpha)$-lexicographic fairness according to Definition~\ref{def:lexifair}.
\end{clm}

\begin{proof}
If a model satisfies $(\ell,\alpha)$-lexicographic fairness according to Definition~\ref{def:failedlexifair}, then by taking $\epsvec = \vec{0}$, it also meets the conditions of Definition~\ref{def:lexifair}.
\end{proof}

We now face another definitional challenge. A priori, Definition \ref{def:lexifair} appears to be highly non-convex, because it constrains the second highest group error, the the third highest group error, etc.\footnote{E.g., if we have two groups and two models which induce group errors   $(0.5,0)$ and $(0, 0.5)$ respectively, both solutions have a second-highest error of 0 --- but  convex combinations have a second highest error strictly greater than 0. So absent other structure, upper bounding the second highest group error of a model corresponds to a non-convex constraint. But note that in this two-group example, the non-convexity dissapears if we restrict attention to minimax optimal models. This is what we will take advantage of more generally.} This is in contrast to standard equal-error notions of fairness, or minimax fairness (which constrains only the highest group error) that \emph{are} convex in the sense that a distribution over fair models remains fair. Without convexity of this sort, the algorithmic problem of finding a fair model becomes much more challenging. But in fact (at least for $\alpha = 0$), Definition \ref{def:lexifair} \emph{does} give a convex constraint. To see this, we first introduce an alternative notion of \emph{convex lexifairness}, and then show that it actually represents the exact same constraint as lexifairness when the approximation parameter $\alpha = 0$. 

\begin{definition}[Convex Lexicographic Fairness]\label{def:convexlexifair}
    Fix a distribution $\Ps$. Let $1 \le \ell \le K$ and $\alpha \ge 0$. For any sequence of mappings $\epsvec = \left( \epsilon_1, \epsilon_2, \ldots, \epsilon_\ell \right)$ where $\epsilon_j \in \R^{\Hs}$, define $\F^{\epsvec}_{(0)} (\Ps) \triangleq \Hs$, and recursively for all $1 \le j \le \ell$ define:
    \[
    \F_{(j)}^{\epsvec} (\Ps) \triangleq \left\{ h \in \F_{(j-1)}^{\epsvec}(\Ps): \max_{ \left\{ i_1, \ldots, i_j \right\} \subseteq [K] } \sum_{r=1}^j L_{i_r} (h, \Ps) \le \min_{g \in \F_{(j-1)}^{\epsvec}(\Ps) } \max_{\left\{ i_1, \ldots, i_j \right\} \subseteq [K]} \sum_{r=1}^j L_{i_r} (g, \Ps) + \epsilon_j (h) \right\}
    \]
    and let $\Vert \epsvec \Vert_\infty = \max_{1 \le j \le \ell} \max_{h \in \Hs} \epsilon_j (h)$. We say a model $h \in \Hs$ satisfies $(\ell, \alpha)$-convex lexicographic fairness with respect to $\Ps$ if there exists $\epsvec$ with $\Vert \epsvec \Vert_\infty \le \alpha$ such that for all  $j \le \ell$:
    \[
    \max_{ \left\{ i_1, \ldots, i_j \right\} \subseteq [K] } \sum_{r=1}^j L_{i_r} (h, \Ps) \le \min_{g \in \F_{(j-1)}^{\epsvec}(\Ps) } \max_{\left\{ i_1, \ldots, i_j \right\} \subseteq [K]} \sum_{r=1}^j L_{i_r} (g, \Ps) + \epsilon_j (h) + \alpha.
    \]
    When we prove bounds on empirical convex lexifairness, we simply take the distribution to be the uniform distribution over the data set $S$. When the distribution is clear from context, we will write $ \F_{(j)}^{\epsvec}$ and elide the dependence on the distribution. 
\end{definition}

Here, we have replaced constraints on the $j$'th highest group error with constraints on the \emph{sum} of group errors over all $\approx K^j$ subsets of groups of size $j$. This has replaced a single constraint with many constraints, but each is convex, and hence the resulting set of constraints defined by $\F_{(j)}^{\epsvec}$ is convex. We will formally prove this in the following claim.

\begin{restatable}[Convexity of $\F_{(j)}^{\epsvec}$]{clm}{convexity}\label{clm:convexity}
Let $L_z: \Hs \to \R_{\ge 0}$ be a convex loss function. If the initial model class $\Hs$ is convex, then for all $j$ and all $\epsvec$ such that the mappings $\epsilon_j \in \R^\Hs$ are concave, the set $\F_{(j)}^{\epsvec}$ is convex.
\end{restatable}

The proof can be found in Appendix \ref{app:clm}, and proceeds by straightforward induction. We note that while some classes of models naturally satisfy the convexity conditions of the above claim with respect to their corresponding parameters (e.g. linear and logistic regression), this claim will apply to arbitrary classification models with zero-one loss as well. In these settings, we will convexify the class of models by considering the set of all probability distributions over deterministic models. The loss of a distribution (i.e. a randomized model) is then defined as the \emph{expected} loss, when the model is sampled from the corresponding distribution. Hence, by linearity of expectation, our loss functions will be convex (linear) in the parameters --- i.e.~the weights --- of these distributions.

It turns out that our notion of \emph{convex} lexifairness is identical to our notion of lexifairness (and so our original definition in fact specified a convex set of constraints), at least when the approximation parameter $\alpha = 0$. We prove this in the following claim:
\begin{restatable}[Relationship between $\F_{(j)}^{\epsvec}$ and $\Hs_{(j)}^{\epsvec}$ when $\epsvec = \vec{0}$]{clm}{relationship}\label{clm:relationship}
    For all $j$, and $\epsvec = \vec{0}$, we have $\F_{(j)}^{\epsvec} = \Hs_{(j)}^{\epsvec}$.
\end{restatable}

 The intuition for the claim is the following. The sets $\Hs_{(j)}$ in Definition \ref{def:lexifair} constrain the error of the group that has the $j$'th highest error. In contrast, the sets $\F_{(j)}$ from Definition \ref{def:convexlexifair} constrain the \emph{sum} of the errors for all possible $j$-tuples of groups. Amongst all of these constraints, the binding one will be the constraint corresponding to the $j$ groups that have the \emph{largest} errors. But because (inductively) the errors of the top $j-1$ error groups have already been appropriately constrained in $\F_{(j-1)}$, this reduces to a constraint on the $j$'th highest error group, as desired. These constraints are numerous, but each is convex, and so the resulting set of constraints can be seen to be convex. See Appendix \ref{app:clm} for the full proof of Claim \ref{clm:relationship}, which proceeds by induction. 

We emphasize that despite the complexity of our final Definition \ref{def:convexlexifair}, what we have shown is that it is in fact a relaxation of our initial, natural definition of exact lexifairness (Definition \ref{def:exactlexifair}) --- and in particular Definitions \ref{def:exactlexifair}, \ref{def:lexifair}, and \ref{def:convexlexifair} coincide exactly when $\alpha = 0$. 
We do not know the precise relationship between our definitions of approximate lexifairness and approximate convex lexifairness for $\alpha > 0$ --- but because both are smooth relaxations of the same base definition, both should be viewed as capturing the same intuition as Definition \ref{def:exactlexifair} (exact lexifairness) when $\alpha$ is small. 
\section{Game Theory and No-Regret Learning Preliminaries}
\subsection{No-Regret Dynamics}\label{subsec:noregret}
In this subsection, we briefly review the seminal result of Freund and Schapire \cite{fs1996}: Under certain conditions, two-player zero-sum games can be (approximately) solved by having access to a no-regret online learning algorithm for one of the players.

Suppose in this subsection that $S_1$ and $S_2$ are two vector spaces over the field of real numbers. Consider a zero-sum game with two players: a player with strategies in $S_1$ (the minimization player) and another player with strategies in $S_2$ (the maximization player). Let $U: S_1 \times S_2 \to \R_{\ge 0}$ be the payoff function of this game. For every strategy $s_1 \in S_1$ of player one and every strategy $s_2 \in S_2$ of player two, the first player gets utility $-U(s_1, s_2)$ and the second player gets utility $U(s_1, s_2)$.
\begin{definition}[Approximate Equilibrium]\label{def:nuapprox}
A pair of strategies $(s_1, s_2) \in S_1 \times S_2$ is said to be a $\nu$-approximate minimax equilibrium of the game if the following conditions hold:
\[
 U(s_1, s_2) - \min_{s'_1 \in S_1}  U(s'_1, s_2)  \le \nu,
\quad
\max_{s'_2 \in S_2}  U(s_1, s'_2) -  U(s_1, s_2)  \le \nu
\]
\end{definition}

In other words, $(s_1, s_2)$ is a $\nu$-approximate equilibrium of the game if neither player can gain more than $\nu$ by deviating from their strategies.

Freund and Schapire \cite{fs1996} proposed an efficient framework for approximately solving the game: In an iterative fashion, have one of the players play according to a no-regret learning algorithm, and let the second player (approximately) best respond to the play of the first player. The empirical average of each player's actions over a sufficiently long sequence of such play will form an approximate equilibrium of the game. The formal statement is given in the following theorem.

\begin{thm}[No-Regret Dynamics \cite{fs1996}]\label{thm:noregret}
    Let $S_1$ and $S_2$ be convex, and suppose the utility function $U$ is convex-concave: $U(\cdot, s_2): S_1 \to \R_{\ge 0}$ is convex for all $s_2 \in S_2$, and $U(s_1, \cdot): S_2 \to \R_{\ge 0}$ is concave for all $s_1 \in S_1$. Let $(s_1^1, s_1^2, \ldots, s_1^T)$ be the sequence of  play for the first player, and let $(s_2^1, s_2^2, \ldots, s_2^T)$ be the sequence of play for the second player. Suppose for $\nu_1,\nu_2 \ge 0$, the regret of the players jointly satisfies
    \[
    \sum_{t=1}^T U(s_1^t, s_2^t) - \min_{s_1 \in S_1} \sum_{t=1}^T U(s_1, s_2^t) \le \nu_1 T,
    \quad
    \max_{s_2 \in S_2} \sum_{t=1}^T U(s_1^t, s_2) - \sum_{t=1}^T U(s_1^t, s_2^t) \le \nu_2 T
    \]
    Let $\bar{s}_1 = \frac{1}{T}\sum_{t=1}^T s_1^t \in S_1$ and $\bar{s}_2 = \frac{1}{T}\sum_{t=1}^T s_2^t \in S_2$ be the empirical average play of the players. We have that the pair $(\bar{s}_1, \bar{s}_2)$ is a $(\nu_1+\nu_2)$-approximate equilibrium of the game. 
\end{thm}

\emph{No regret} online learning algorithms are algorithms that can guarantee the conditions of Theorem \ref{thm:noregret} against arbitrary adversaries. We will use two no-regret online learning algorithms: \emph{Online Projected Gradient Descent}, which we will use in regression settings in which models are represented by parameters in a Euclidean space, and \emph{Follow the Perturbed Leader (FTPL)}, which we will use in binary classification settings. We will make use of these no-regret learning algorithms in our proposed algorithm for learning a lexifair model; full explanations and pseudocode for both are in Appendix \ref{app:gd+ftpl}.

\section{Finding Lexifair Models}
\label{sec:finding}
In this section we focus on developing the tools required to prove the following (informally stated) theorem. The formal claims are provided in Theorems~\ref{thm:lexifair-reg} and \ref{thm:lexifair-clf}.
\begin{thm}[Informal]
Suppose the model class $\Hs$ is convex and compact, and that the loss function $L_z: \Hs \to \R_{\ge 0}$ is convex for all data points $z \in \Z$. There exists an efficient algorithm that returns a model which is $(\ell, \alpha)$-convex lexicographic fair (according to Definition~\ref{def:convexlexifair}), for any given $\ell$ and $\alpha$. 
\end{thm}

We will propose algorithms for both classification and regression settings. The algorithms we propose proceed inductively to solve the minimax problems defined recursively by our convex lexifair definition. The first minimax problem is the one that minimizes the maximum group error rate: $\min_{h \in \Hs} \max_{k \in [K]} L_k (h)$. Let us denote the estimated value (computed by the first phase of our algorithm) for this minimax problem by $\eta_1$. The second minimax problem is minimizing the maximum sum of any two group error rates subject to the constraint that all group error rates are at most $\eta_1$: the estimated value for this minimax problem  is called $\eta_2$. The rest of the minimax problems are defined in a similar inductive fashion: suppose at round $j \le \ell$, we are given some estimates $(\eta_1, \ldots, \eta_{j-1})$ for the first $j-1$ minimax values.  Now using these estimates, the new minimax problem for the sum of any $j$ group error rates can be stated as follows.
\begin{equation}\label{eq:opt-problem-old}
\min_{\overset{h \in \Hs:}{\overset{\forall r \le j-1, \, \forall \{i_1, \ldots, i_r\} \subseteq [K]}{L_{i_1} (h) + \ldots + L_{i_r} (h) \le \eta_r}}} \left\{ \max_{\left\{ i_1, \ldots, i_j \right\} \subseteq [K]} \sum_{r=1}^j L_{i_r} (h) \right\}
\end{equation}
We can reformulate this optimization problem by calling the objective $\max_{\left\{ i_1, \ldots, i_j \right\} \subseteq [K]} \sum_{r=1}^j L_{i_r} (h) := \eta_j$ and introducing a new set of constraints which require that any sum of $j$ group error rates must be at most $\eta_j$. Note that this new formulation introduces a new variable, $\eta_j$, to the optimization problem. We therefore have that the optimization problem~\eqref{eq:opt-problem-old} is equivalent to
\begin{equation}\label{eq:opt-problem}
\min_{\overset{h \in \Hs, \eta_j \in [0,j \cdot L_M]:}{\overset{\forall r \le j, \, \forall \{i_1, \ldots, i_r\} \subseteq [K]}{L_{i_1} (h) + \ldots + L_{i_r} (h) \le \eta_r}}} \eta_j
\triangleq \opt_j \left( \eta_1, \ldots, \eta_{j-1} \right)
\end{equation}

which is a constrained convex optimization problem given that the model class $\Hs$ and the loss function $L$ are convex. Here $L_M = \max_{z,h} L_z(h)$ is an upper bound on the loss function which identifies the range of feasible values for $\eta_j$: $[0, j \cdot L_M]$. Recall that in this round, $(\eta_1, \ldots, \eta_{j-1})$ are given from the previous rounds, and $\eta_j$ is a variable in the optimization problem. We denote the optimal value of the optimization problem~\eqref{eq:opt-problem} by $\opt_j \left( \eta_1, \ldots, \eta_{j-1} \right)$.

\subsection{Formulation as a Two-Player Zero-Sum Game}
Optimization problem \eqref{eq:opt-problem} is written as a constrained optimization problem, but we can express it equally well as an unconstrained minimax problem via Lagrangian duality. The corresponding Lagrangian can be written as:
\begin{equation}\label{eq:lagrangian}
\mathcal{L}_j \left( (h, \eta_j), \lambda \right) = \eta_j + \sum_{r=1}^j \sum_{\{i_1, \ldots, i_r\} \subseteq [K]} \lambda_{\{i_1, i_2, \ldots, i_r\}} \cdot \left( L_{i_1} (h) + \ldots + L_{i_r} (h) - \eta_r \right)
\end{equation}
where we introduce one dual variable $\lambda$ for every inequality constraint in the optimization problem~\eqref{eq:opt-problem}, and index the dual variables by their corresponding constraint. Therefore, there are $q_j =\sum_{r=1}^j \binom{K}{r}$ dual variables in this round. Solving optimization problem~\eqref{eq:opt-problem} is equivalent to solving the following minimax problem:
\begin{equation}\label{eq:minimax}
    \min_{h \in \Hs, \eta_j \in [0,j \cdot L_M]} \max_{\lambda \in \R^{q_j}_{\ge 0}} \mathcal{L}_j \left( (h, \eta_j), \lambda \right) = \max_{\lambda \in \R^{q_j}_{\ge 0}} \min_{h \in \Hs, \eta_j \in [0,j \cdot L_M]} \mathcal{L}_j \left( (h, \eta_j), \lambda \right)
\end{equation}
where the minimax theorem holds because 1) the range of the primal variables, i.e. $\Hs$ and $[0, j \cdot L_M]$, is convex and compact, the range for the dual variable ($\R^q_{\ge 0}$) is convex, and 2) $\mathcal{L}_j \left( (h, \eta_j), \lambda \right)$ is convex in its primal variables $(h, \eta_j)$ and concave in the dual variable $\lambda$. Therefore we focus on solving the minimax problem \eqref{eq:minimax} which can be seen as solving a two-player zero-sum game with payoff function $\mathcal{L}_j \left( (h, \eta_j), \lambda \right)$. Using the no-regret dynamics of \cite{fs1996} (see Section \ref{subsec:noregret}), we will have the primal player (or \emph{Learner}) with strategies $(h, \eta_j) \in \Hs \times [0, j \cdot L_M]$ play a no-regret learning algorithm and let the dual player (or \emph{Auditor}) with strategies $\lambda \in \Lambda_j = \{\lambda \in \R_{\ge 0}^{q_j} : \Vert \lambda \Vert_1 \le B \}$ best respond. Here we place an upper bound $B$ on the $\ell_1$-norm of the dual variable to guarantee convergence of our algorithms. This nuisance parameter  will be set optimally in our algorithms, and we note that the minimax theorem continues to hold in the presence of this upper bound on $\lambda$. We will first analyze the best response problem for both players --- i.e. the problem of optimizing the Lagrangian for one of the players \emph{fixing} the strategy of the other player.

\subsection{The Auditor's Best Response}
Fixing the $(h, \eta_j)$ variables of the Learner and the estimated values $(\eta_1, \ldots, \eta_{j-1})$ from previous rounds, the Auditor can best respond by solving
\[
\argmax_{\lambda \in \Lambda_j}  \mathcal{L}_j \left( (h, \eta_j), \lambda \right) \equiv \argmax_{\lambda \in \Lambda_j} \sum_{r=1}^{j} \sum_{\{i_1,\ldots,i_r\} \subseteq [K]} \lambda_{\{i_1,i_2, \ldots,i_r\}} \cdot \left( L_{i_1} (h) + \ldots + L_{i_r} (h) - \eta_r \right)
\]
Since the objective is linear in the dual variables $\lambda$, the Auditor can without loss of generality best respond by putting all its mass $B$ on the variable $\lambda_{\{i_1,i_2, \ldots,i_r\}}$ corresponding to the most violated constraint, if one exists. In particular, given any model $h \in \Hs$ and any ordering $\bar{h}$ induced by $h$ on the groups, we have that the Auditor's best response $\lambda_\text{best} (h, \eta_j)$ is
\[
\lambda_\text{best} (h, \eta_j) =
\begin{cases}
0 \in \R^{q_j} & \text{if }\forall r \le j: \, L_{\bar{h}(1)} (h) + \ldots + L_{\bar{h}(r)} (h) \le \eta_r \\
\lambda^\star \in \R^{q_j} & \text{if }\exists r \le j: \, L_{\bar{h}(1)} (h) + \ldots + L_{\bar{h}(r)} (h) > \eta_r
\end{cases}
\]
where the entries of $\lambda^\star$ are defined as follows.
\begin{equation}\label{eq:lambdastar}
\lambda^\star_{\{i_1,i_2,\ldots,i_r\}} =
\begin{cases}
B & \text{if } \{i_1,i_2,\ldots,i_r\} = \{ \bar{h}(1), \bar{h}(2), \ldots, \bar{h}(r^\star) \} \\
0 & \text{Otherwise}
\end{cases}
\end{equation}
where $r^\star \in \argmax_{r \le j} \left( L_{\bar{h}(1)} (h) + \ldots + L_{\bar{h}(r)} (h) - \eta_r \right)$.

Note that the Auditor's best response can be computed efficiently because it only requires sorting the vector of error rates across $K$ groups. We summarize the best response algorithm for the Auditor in Algorithm~\ref{alg:auditorbest}.

\begin{algorithm}[t]
\KwIn{Learner's play $(h, \eta_j)$, previous estimates $(\eta_1, \ldots, \eta_{j-1})$}
Compute $L_k (h)$ for all groups $k \in [K]$\;
Find the top $j$ elements of vector $(L_1 (h), \ldots, L_K (h))$ and call them: $L_{\bar{h}(1)} (h) \ge \ldots \ge L_{\bar{h}(j)} (h)$\;
\lIf{$\forall r \le j: \, L_{\bar{h}(1)} (h) + \ldots + L_{\bar{h}(r)} (h) \le \eta_r$}{
$\lambda_{out} = 0$}
%\If{$\exists r \le j: \, L_{\bar{h}(1)} (h) + \ldots + L_{\bar{h}(r)} (h) > \eta_r$}{
\lElse{ Let $r^\star \in \argmax_{r \le j} \left( L_{\bar{h}(1)} (h) + \ldots + L_{\bar{h}(r)} (h) - \eta_r \right)$, $\lambda_{out} = \lambda^\star$ as in Equation~\eqref{eq:lambdastar}
}
\KwOut{$\lambda_{out} \in \Lambda_j$}
\caption{The Auditor's Best Response ($\lambda_\text{best}$): $j$th round}
\label{alg:auditorbest}
\end{algorithm}

\subsection{The Learner's Best Response}
Given dual weights $\lambda \in \Lambda_j$ chosen by the Auditor, the Learner can best respond by solving
\[
\argmin_{h \in \Hs, \eta_j \in [0, j \cdot L_M]} \mathcal{L}_j \left( (h, \eta_j), \lambda \right).
\]
We note that the objective function $\mathcal{L}_j \left( (h, \eta_j), \lambda \right)$ can be decomposed into three terms: one that depends only on the model $h$, another that depends only on $\eta_j$, and finally one that is constant (with respect to $(h, \eta_j)$). Therefore, this optimization problem is separable for the Learner --- the decomposition is formally described below.
\begin{equation}\label{eq:decomposition}
\mathcal{L}_j \left( (h, \eta_j), \lambda \right) = \mathcal{L}_j^1 \left( h, \lambda \right) + \mathcal{L}_j^2 \left( \eta_j, \lambda \right) + C_j \left( \lambda \right)
\end{equation}
where
\begin{equation}\label{eq:learner-p}
\mathcal{L}_j^1 \left( h, \lambda \right) \triangleq \sum_{r=1}^K w_r (\lambda) L_r (h), \ \text{where} \  w_r (\lambda) \triangleq \sum_{s=0}^{j-1} \sum_{\{i_2,\ldots,i_s\} \subseteq [K] \setminus \{r\}} \lambda_{\{r,i_2,\ldots,i_s\}}
\end{equation}
\begin{equation}\label{eq:learner-eta}
\mathcal{L}_j^2 \left( \eta_j, \lambda \right) \triangleq \left( 1 - \sum_{\{i_1, \ldots, i_j\} \subseteq [K]} \lambda_{\{i_1, i_2, \ldots, i_j\}} \right) \eta_j
\end{equation}
\begin{equation}\label{eq:constant}
C_j \left( \lambda \right) \triangleq - \sum_{r=1}^{j-1} \sum_{\{i_1, \ldots, i_r\} \subseteq [K]} \lambda_{\{i_1, i_2, \ldots, i_r\}} \cdot \eta_r
\end{equation}
Given this decomposition of the Lagrangian, the best response $(h,\eta_j)$ of the Learner to the variables $\lambda$ of the Auditor is as follows:
\[
(h,\eta_j) = \left( \argmin_{h \in \Hs} \mathcal{L}_j^1 \left( h, \lambda \right) ,  \argmin_{\eta_j \in [0, j\cdot L_M]} \mathcal{L}_j^2 \left( \eta_j, \lambda \right) \right)
\]

Note that the first optimization problem is a weighted minimization problem over the class $\Hs$, and the second one is a simple minimization of a linear function. Furthermore, even though in general computing the sums in Equations~\eqref{eq:learner-p} and \eqref{eq:learner-eta} can be computationally hard (because they are sums over exponentially many terms), \emph{when the Auditor is best responding (which will be the case in our algorithms), these sums can be computed efficiently}. We formally state this claim in Fact~\ref{fact:efficient}.

\begin{fact}\label{fact:efficient}
When the Auditor is using its best response algorithm (Algorithm~\ref{alg:auditorbest}) to respond to the Learner, the Auditor will either output zero or identify a single subset $C$ of groups ($\vert C \vert \le j$) on which the constraints are violated maximally. In the former case, $w_r (\lambda) = 0$ for all $r$ and $1 - \sum_{\{i_1, \ldots, i_j\} \subseteq [K]} \lambda_{\{i_1, i_2, \ldots, i_j\}} = 1$. In the latter case, we have
\[
w_r (\lambda)=B \cdot \1 \left[ r \in C \right], \quad 1 - \sum_{\{i_1, \ldots, i_j\} \subseteq [K]} \lambda_{\{i_1, i_2, \ldots, i_j\}} = 1 - B \cdot \1 \left[ \vert C \vert = j \right]
\]
\end{fact}

\subsection{Solving the Game with No-Regret Dynamics}
Having analyzed the best response problem for both players, we now focus on developing efficient algorithms to approximately solve the two-player zero-sum game defined above, which corresponds to finding an approximate convex lexifair model. The algorithms we propose use  no-regret dynamics (see Section~\ref{subsec:noregret}) in which the Learner plays a no-regret learning algorithm and the Auditor best responds according to Algorithm~\ref{alg:auditorbest}. As a consequence, we get that the empirical average of the played strategies $((\hat{h}, \hat{\eta}_j), \hat{\lambda})$ of the players over the course of the iterative algorithms will form a $\nu$-approximate equilibrium of the game for some small value of $\nu \ge 0$ (according to Definition~\ref{def:nuapprox}). Then, by the following theorem, we can turn these equilibrium guarantees into the fairness guarantees of the output model $\hat{h}$. Its proof can be found in Appendix \ref{app:finding}. 

We remark that what we mean by the empirical average will depend on the setting. If we are in a setting in which the loss function is convex in the model parameters (e.g. logistic or linear regression), then we can actually average the model parameters, and output a single deterministic model. Alternately, if we are in a classification setting in which the loss function (e.g. zero-one loss) is non-convex in the model parameters, then by averaging, we mean using the randomized model that corresponds to the uniform distribution over the empirical play history. 
\begin{restatable}{thm}{thmdunno}
\label{thm:dunno}
    At round $j$, let $(\hat{\eta}_1, \ldots, \hat{\eta}_{j-1})$ be any given estimated minimax values from the previous rounds and let the strategies $((\hat{h}, \hat{\eta}_j), \hat{\lambda})$ form a $\nu$-approximate equilibrium of the game for this round, i.e.,
    \[
    \mathcal{L}_j \left( (\hat{h}, \hat{\eta}_j), \hat{\lambda} \right) \le \min_{h \in \Hs, \eta_j \in [0, j\cdot L_M]} \mathcal{L}_j \left( (h, \eta_j), \hat{\lambda} \right) + \nu, \quad \mathcal{L}_j \left( (\hat{h}, \hat{\eta}_j), \hat{\lambda} \right) \ge \max_{\lambda \in \Lambda_j} \mathcal{L}_j \left( (\hat{h}, \hat{\eta}_j), \lambda \right) - \nu
    \]
    We have that
\begin{equation*}
\hat{\eta}_j \le OPT_j \left( \hat{\eta}_1, \ldots, \hat{\eta}_{j-1} \right) + 2\nu
\end{equation*}
and for all $r \le j$,
\begin{equation*}
\max_{\left\{ i_1, \ldots, i_r \right\} \subseteq [K]} \sum_{s=1}^r L_{i_r} (\hat{h}) \le \hat{\eta}_{r} + \frac{j L_M + 2 \nu}{B}.
\end{equation*}
\end{restatable}

We will next instantiate this general result to give concrete algorithms for learning convex lexifair models in the regression and classification settings respectively. 

\section{Finding Lexifair Regression Models}
\label{sec:reg}
Suppose in this section that $\Y \subseteq \R$ and $\Hs$ is a class of models in which each model is parametrized by some $d$-dimensional vector in $\R^d$: $\Hs = \left\{ h_\theta: \theta \in \Theta \right\}$ where $\Theta \subseteq \R^d$. In this parametric setting we can think of each parameter $\theta \in \Theta$ as a model and write the loss function as a function of $\theta$. Suppose the loss function $L_z: \Theta \to \R_{\ge 0}$ is differentiable for all $z$.\footnote{If it is not differentiable we can use sub-gradients instead of gradients.} We will have the Learner play according to the Online Projected Gradient Descent algorithm (see Appendix~\ref{subsec:gd}) where the gradients of the corresponding loss function of the game for the Learner (i.e. $\mathcal{L}_j \left( (\theta, \eta_j), \lambda \right)$) can be computed using Equations~\eqref{eq:learner-p} and \eqref{eq:learner-eta}, and the decomposition given in~\eqref{eq:decomposition}:

\begin{equation}\label{eq:grad-model}
\nabla_\theta \mathcal{L}_j \left( (\theta, \eta_j), \lambda \right) = \nabla_\theta \mathcal{L}_j^1 \left( \theta, \lambda \right) = \sum_{r=1}^K w_r (\lambda) \nabla_\theta L_r (\theta),
\end{equation}
\begin{equation}\label{eq:grad-eta}
\nabla_{\eta_j} \mathcal{L}_j \left( (\theta, \eta_j), \lambda \right) = \nabla_{\eta_j} \mathcal{L}_j^2 \left( \eta_j, \lambda \right) = 1 - \sum_{\{i_1, \ldots, i_j\} \subseteq [K]} \lambda_{\{i_1, i_2, \ldots, i_j\}}.
\end{equation}
The algorithm for this setting is given as Algorithm~\ref{alg:fair-regression}, which makes calls to a subroutine (Algorithm~\ref{alg:nr-regression}) that solves the two-player zero-sum games defined above by having the Learner play Online Projected Gradient Descent (see Appendix ~\ref{app:gd+ftpl}) and the Auditor best respond using Algorithm~\ref{alg:auditorbest}. Note that since the Auditor is best responding, computing the sums in Equations~\eqref{eq:grad-model} and \eqref{eq:grad-eta} can be done efficiently per Fact~\ref{fact:efficient}.

\begin{restatable}[Lexifairness for Regression]{thm}{regthm}\label{thm:lexifair-reg}
    Suppose $\Theta \subseteq \R^d$ is convex, compact, and bounded with diameter $D$: $\sup_{\theta, \theta' \in \Theta} \left\Vert \theta - \theta' \right\Vert_2 \le D$. Suppose the loss function $L_z: \Theta \to \R_{\ge 0}$ is convex and that there exists constants $L_M$ and $G$ such that $L_z (\cdot) \le L_M$ and $\Vert \nabla_\theta L_z (\cdot) \Vert_2 \le G$, for all data points $z \in \Z$. We have that for any $\ell \le K$ and any $\alpha \ge 0$, the model $\hat{\theta}_\ell \in \Theta$ output by Algorithm~\ref{alg:fair-regression} is $(\ell,\alpha)$-convex lexicographic fair.
\end{restatable}
The proof of this theorem (which can be found in Appendix \ref{app:reg}) involves bounding the regret of each player, and then appealing to Theorem \ref{thm:dunno}.

\begin{algorithm}[t]
\KwIn{$S = \cup_{k=1}^K G_k$ data set consisting of $K$ groups, $(\ell,\alpha)$ desired fairness parameters, loss function parameters $L_M$ and $G$, diameter $D$ of the model class $\Theta$}
\For{$j=1,2, \ldots, \ell$}{
Set $T_j = \frac{4 j^2 (GD + L_M)^2 (2 \alpha + j L_M)^2}{\alpha^4}$\;
Set $B_j = \frac{\alpha + j L_M}{\alpha}$\;
$(\hat{\theta}_j, \hat{\eta}_j) = \mathtt{RegNR} (T_j, B_j; \hat{\eta}_1, \ldots, \hat{\eta}_{j-1})$ (Calling Algorithm~\ref{alg:nr-regression})
}
\KwOut{$(\ell,\alpha)$-convex lexifair model $\hat{\theta}_\ell$} 
\caption{$\mathtt{LexiFairReg}$: Finding a Lexifair Regression Model}
\label{alg:fair-regression}
\end{algorithm}

\begin{algorithm}[t]
\KwIn{Number of rounds $T$, dual variable upper bound $B$, previous estimates $(\eta_1, \ldots, \eta_{j-1})$}
%Initialize the dual player $\lambda^0 \in \Lambda$\;
Set learning rates $\eta = \frac{D}{jBG\sqrt{T}}$ and $\eta' = \frac{jL_M}{(1+B)\sqrt{T}}$\;
Initialize the Learner: $\theta^1 \in \Theta, \eta_j^1 \in [0, j \cdot L_M]$\;
\For{$t=1,2, \ldots, T$}{
Learner plays $(\theta^t, \eta_j^t)$\;
Auditor best responds: $\lambda^t = \lambda_\text{best} (\theta^t, \eta_j^t ; (\eta_1, \ldots, \eta_{j-1}))$ using Algorithm~\ref{alg:auditorbest}\;
Learner updates its actions using Projected Gradient Descent:
\[
\theta^{t+1} = \text{Proj}_\Theta \left( \theta^{t} - \eta \cdot \nabla_\theta \mathcal{L}_j (\theta^{t}, \eta_j^{t}, \lambda^{t}) \right)
\]
\[
\eta_j^{t+1} = \text{Proj}_{[0, j \cdot L_M]} \left(\eta_j^{t} - \eta' \cdot \nabla_{\eta_j} \mathcal{L}_j (\theta^{t}, \eta_j^{t}, \lambda^{t}) \right)
\]
where the gradients are given in Equations~\eqref{eq:grad-model} and \eqref{eq:grad-eta}.
}
\KwOut{the average play $\hat{\theta} = \frac{1}{T} \sum_{t=1}^T \theta^t \in \Theta$, and $\hat{\eta}_j = \frac{1}{T} \sum_{t=1}^T \eta_j^t \in [0, j\cdot L_M]$.}
\caption{$\mathtt{RegNR}$: $j$th round}
\label{alg:nr-regression}
\end{algorithm}

\section{Finding Lexifair Classification Models}
\label{sec:class}
Suppose in this section that $\Y = \{0,1\}$ and our model class $\Hs$ is the probability simplex over a class of deterministic binary classifiers. We slightly abuse notation and write $\Hs$ for the given class of deterministic classifiers and write $\Delta \Hs \triangleq \{ p: p \text{ is a distribution over } \Hs \}$ for the probability simplex, and work with $\Delta \Hs$ as our model class. Let the loss function be zero-one loss: for any $h \in \Hs$: $L_z (h) = \mathbbm{1} \left\{ h(x) \neq y \right\}$. The loss of any randomized model $p$ on data point $z$ is defined as the \emph{expected loss} of $h$ on $z$ when $h$ is sampled from $\Hs$ according to the distribution $p$. In other words,
\[
L_z( p ) \triangleq \E_{h \sim p} \left[ L_z (h )\right],
\]
which is convex (linear) in the model $p$ (weights of the distribution). We will also assume that the model class $\Hs$ has finite VC dimension: $d_\Hs < \infty$. Sauer's Lemma below will then imply that for any finite dataset, $\Hs$ induces only finitely many labelings. This will serve two purposes. First, it allows us to write the optimization problem as a linear program with \emph{finitely many} variables, and therefore appeal to strong duality. Second, it allows us to pose the Learner's best response problem as an $n$-dimensional \emph{linear optimization} problem, over the only exponentially many labelings of the $n$ data points. This is what will allow us to apply Follow the Perturbed Leader and obtain oracle-efficient no-regret learning guarantees for the Learner. Here we are following an approach similar to that of \cite{kearns2018preventing}. 

\begin{lemma}[Sauer's Lemma]\label{lem:sauer}
Let $S = \{z_i = (x_i,y_i)\}_{i=1}^n$ be a data set of size $n$ and $\Hs$ be a model class with VC dimension $d_\Hs$. Let $\Hs (S) \triangleq \left\{ \left( h(x_1), h(x_2), \ldots, h(x_n) \right): h \in \Hs \right\}$ be the set of all labelings induced by $\Hs$ on data set $S$. We have that $\vert \Hs (S) \vert = O(n^{d_\Hs})$.
\end{lemma}

Recall that given some $\lambda \in \Lambda_j$ of the Auditor, the best response of the Learner is separable and given by
\[
(p,\eta_j) = \left( \argmin_{p \in \Delta \Hs} \mathcal{L}_j^1 \left( p, \lambda \right) ,  \argmin_{\eta_j \in [0, j]} \mathcal{L}_j^2 \left( \eta_j, \lambda \right) \right),
\]
where $\mathcal{L}_j^1 \left( p, \lambda \right)$ and $\mathcal{L}_j^2 \left( \eta_j, \lambda \right)$ are given in Equations~\eqref{eq:learner-p} and \eqref{eq:learner-eta}, respectively, and we use $L_M = 1$ because our loss function in this section is the zero-one loss. We can now apply Sauer's Lemma and the fact that $\mathcal{L}_j^1 \left( p, \lambda \right)$ is linear in the weights of the distribution $p$ to rewrite the first optimization problem as
\[
\argmin_{p \in \Delta \Hs} \mathcal{L}_j^1 \left( p, \lambda \right) \quad \underset{\longrightarrow}{\text{Sauer's}} \quad \argmin_{p \in \Delta \Hs (S)} \mathcal{L}_j^1 \left( p, \lambda \right) \quad \underset{\longrightarrow}{\text{linearity}} \quad \argmin_{h \in \Hs (S)} \mathcal{L}_j^1 \left( h, \lambda \right),
\]
which is an optimization problem over finitely many variables (weights over $\Hs(S)$). Note we can further rewrite this optimization problem as a cost sensitive classification problem which can be solved by calling a \emph{Cost Sensitive Classification Oracle} for $\Hs$ (CSC($\Hs$)). Recall from Equation~\eqref{eq:learner-p} that $\mathcal{L}_j^1 \left( h, \lambda \right) = \sum_{r=1}^K w_r (\lambda) L_r (h)$. We have that
\[
\argmin_{h \in \Hs (S)} \sum_{r=1}^K w_r (\lambda) L_r (h)  \equiv \argmin_{h \in \Hs (S)} \sum_{i=1}^n \left\{ c^1_i (\lambda) h(x_i) + c_i^0 (\lambda) \left( 1 - h(x_i) \right) \right\}
\]
where
$
c_i^1 (\lambda) = (1-y_i) \sum_{r=1}^K (w_r (\lambda)/n_r) \1 \left\{ i \in G_r \right\}
$
is the cost of classifying data point $i$ as a positive (1) example, and
$
c_i^0 (\lambda) = y_i \sum_{r=1}^K (w_r (\lambda) / n_r) \1 \left\{ i \in G_r \right\}
$
is the cost of classifying data point $i$ as a negative (0) example. Here $n_r$ is the size of the $r$'th group: $n_r = \vert G_r \vert$. By using a linear transformation of the cost vectors, we have 
\[
\argmin_{h \in \Hs (S)} \sum_{r=1}^K w_r (\lambda) L_r (h)  \equiv \argmin_{h \in \Hs (S)} \sum_{i=1}^n c_i (\lambda) h(x_i)
\]
where the vector of costs is given as follows:
\[
\forall 1 \le i \le n: \quad c_i (\lambda) \triangleq \left( 1 - 2 y_i \right) \sum_{r=1}^K \frac{w_r (\lambda)}{n_r} \1 \left\{ i \in G_r \right\}
\]
Let us also define $c(\lambda) \triangleq 1 - \sum_{\{i_1, \ldots, i_j\} \subseteq [K]} \lambda_{\{i_1, i_2, \ldots, i_j\}}$ which is the coefficient of $\eta_j$ in $\mathcal{L}_j^2 \left( \eta_j, \lambda \right)$. We can therefore write the best response of the Learner to $\lambda \in \Lambda_j$ of the Auditor in the classification setting of this section as
\[
(h,\eta_j) = \left( \argmin_{h \in \Hs(S)} \sum_{i=1}^n c_i ( \lambda )  h(x_i) , \, \argmin_{\eta_j \in [0, j]}  c ( \lambda ) \eta_j \right)
\]

Now to get no-regret guarantees for the Learner we use the Follow the Perturbed Leader (FTPL) algorithm (see Appendix~\ref{subsec:ftpl}). At any round $t$ of the two-player zero-sum game, given the history $(\lambda^1, \ldots, \lambda^{t-1})$ (best response plays) of the Auditor, the Learner solves
\[
(h,\eta_j) = \left( \argmin_{h \in \Hs(S)} \sum_{i=1}^n \left( c_i \left( \sum_{s < t} \lambda^s \right) + \frac{1}{\eta} \xi_i \right) h(x_i) ,  \, \argmin_{\eta_j \in [0, j]} \left( c \left(\sum_{s < t} \lambda^s \right) + \frac{1}{\eta'} \xi \right) \eta_j \right)
\]
where $\xi, \xi_i \sim Uniform[0,1]$ for all $i$. At any round $t$, let's denote the true distributions (one over $\Hs(S)$ and another over the interval $[0,j]$) maintained by the Learner's FTPL algorithm by $p^t$ and $D^t$. Since $p^t$ is a distribution over an exponentially large domain ($\vert \Hs (S) \vert = O(n^{d_\Hs})$), we can only represent a sparse version of it efficiently by sampling from it. On the other hand, $D^t$, which is a one dimensional distribution, can be represented by a scaled Bernoulli random variable as follows:
\[
D^t = j \cdot Bern(q^t) \quad \text{where} \quad q^t = \min \left( 1, - \eta' c \left(\sum_{s < t} \lambda^s \right) \cdot \1 \left[ c \left(\sum_{s < t} \lambda^s \right) \le 0 \right] \right)
\]

The algorithm for this setting is given in Algorithm~\ref{alg:fair-classification} which makes calls to a subroutine that implements the no-regret dynamics described above (Algorithm~\ref{alg:nr-classification}). Note that as mentioned earlier, we cannot efficiently represent the FTPL distribution $p^t$ for the Learner, and therefore, we work with the empirical distribution $\hat{p}^t$ of $m$ $i.i.d.$ draws from $p^t$ in Algorithm~\ref{alg:nr-classification}. This makes the best response plays of the Auditor (to the pair $(\hat{p}^t, D^t)$), \emph{approximate} best responses to the actual FTPL distributions $(p^t, D^t)$ of the Learner, and consequently, the Auditor accumulates some regret over the course of the algorithm.

Finally, note that the no-regret dynamics of Algorithm~\ref{alg:nr-classification} must output the average play: $(\bar{p}, \bar{D})$ where $\bar{p} = (1/T) \sum_{t=1}^T p^t$ and $\bar{D} = (1/T) \sum_{t=1}^T D^t$. However, once again we cannot represent the average play $\bar{p}$ efficiently because it can be a distribution over an exponentially large domain. We therefore need to sample from this distribution and take the empirical distribution of this sample as our final output, and this final sampling scheme will introduce additional error on top of the regret of the players. Putting it all together, which requires carefully analyzing the game, including the regret of the players and the additional error due to sampling from $\bar{p}$, results in the following Theorem. Its proof can be found in Appendix \ref{app:class}.

\begin{algorithm}[t]
\KwIn{$S = \cup_{k=1}^K G_k$ data set consisting of $K$ groups, $(\ell,\alpha)$ desired fairness parameters}
Let $n = \vert S \vert$ and $n_{min} = \min_k \vert G_k \vert$\;
\For{$j=1,2, \ldots, \ell$}{
Set $T_j = \frac{256 \left(2 \alpha + j \right)^2 n^3}{\alpha^4 n_{min}^2}$\;
Set $B_j = \frac{\alpha + j}{\alpha}$\;
Set $m_j = \frac{K^2 n_{min}^2 T_j \log \left( 4jKT_j/\delta\right)}{2 n^3}$\;
$(\hat{p}_j, \hat{\eta}_j) = \mathtt{ClfNR} (T_j, B_j, m_j; \hat{\eta}_1, \ldots, \hat{\eta}_{j-1})$ (Calling Algorithm~\ref{alg:nr-classification})
}
\KwOut{$(\ell,\alpha)$-convex lexifair model $\hat{p}_\ell$}
\caption{$\mathtt{LexiFairClf}$: Finding a Lexifair Classification Model}
\label{alg:fair-classification}
\end{algorithm}

\begin{algorithm}[t]
\KwIn{$T$ number of rounds, $B$ dual variable's upper bound, $m$ number of samples to draw, previous estimates $(\eta_1, \ldots, \eta_{j-1})$}
Set learning rates $\eta = \frac{n_{min}}{B} \sqrt{\frac{1}{nT}}$, $\eta' = \frac{1}{1+B}
\sqrt{\frac{1}{T}}$ where $n_{min}$ is the size of the smallest group.\;
Initialize the Learner $\hat{p}^1 \in \Delta \Hs$, $D^{1} \in \Delta ([0,j])$\;
\For{$t=1,2, \ldots, T$}{
Learner plays actions $(\hat{p}^t, D^t)$\;
Auditor Best Responds: $\lambda^t = \lambda_\text{best} (\hat{p}^t, \E_{x \sim D^t} [x])$ using Algorithm~\ref{alg:auditorbest}\;
Update the running sum: $\bar{\lambda}^t = \sum_{s \le t} \lambda^s$\;
Sample from the Learner's FTPL distribution: \\
\For{s = 1,2, \ldots, m}{
Draw $\xi_i \sim U[0,1]$ for all $i \le n$. Call the oracle CSC$(\Hs)$ to solve
\[
h^s = \argmin_{h \in \Hs} \sum_{i=1}^n \left( c_i \left( \bar{\lambda}^t \right) + \frac{1}{\eta} \xi_i \right) h(x_i)
\]
}
Let $\hat{p}^{t+1}$ be the empirical distribution over $\{h^s\}_{s=1}^m$\;
Let $D^{t+1} = j \cdot Bern(q^t)$ where $q^t = \min \left( 1, - \eta' c \left(\bar{\lambda}^t \right) \cdot \1 \left[ c \left(\bar{\lambda}^t \right) \le 0 \right] \right)$\;
}
Sample from the average distribution $\bar{p} = \frac{1}{T} \sum_{t=1}^T p^t$: \\
\For{s = 1,2, \ldots, m}{
Draw a random number $t \in [T]$\;
Draw $\xi_i \sim U[0,1]$ for all $i \le n$. Call the oracle CSC$(\Hs)$ to solve
\[
h^s = \argmin_{h \in \Hs} \sum_{i=1}^n \left( c_i \left( \bar{\lambda}^t \right) + \frac{1}{\eta} \xi_i \right) h(x_i)
\]
}
Let $\hat{p}$ be the empirical distribution over $\{h^s\}_{s=1}^m$\;
Let $\bar{D}$ be the average distribution: $\bar{D} = \frac{1}{T} \sum_{t=1}^T D^t$\;
Let $\hat{\eta}_j = \E_{x \sim \bar{D}} \left[ x \right]$\;
\KwOut{randomized model $\hat{p} \in \Delta \Hs$, and estimate $\hat{\eta}_j \in [0,j]$.}
\caption{$\mathtt{ClfNR}$: $j$th round}
\label{alg:nr-classification}
\end{algorithm}

\begin{restatable}[Lexifairness for Classification]{thm}{mainthm}\label{thm:lexifair-clf}
    Let $\Hs$ be any class of binary classifiers with finite VC dimension, and let $L_z( p ) = \E_{h \sim p} \left[ L_z (h )\right]$ for any randomized model $p \in \Delta \Hs$ where $L_z (h) = \1 \left\{ h(x) \neq y \right\}$ is the zero-one loss. Fix any $\ell \le K$ and any $\alpha \ge 0$. We have that for any $\delta > 0 0$, with probability at least $1-\delta$, the model $\hat{p}_\ell \in \Delta \Hs$ output by Algorithm~\ref{alg:fair-classification} is $(\ell,\alpha)$-convex lexicographic fair.
\end{restatable}

\section{Generalization}
\label{sec:generalization}
In this section, we turn our attention to out of sample bounds. Standard uniform convergence statements would tell us that if we have enough samples from every group, then our in-sample group errors are good estimates of our out of sample group errors. However, this alone does not directly imply that we  satisfy approximate lexifairness out of sample. We prove this is the case below. Our ability to prove out of sample bounds crucially relies on our definitional choices that removed the instability of the naive Definition \ref{def:failedlexifair}. Specifically, we show that if:
\begin{enumerate}
    \item Our base class $\Hs$ satisfies a standard uniform convergence bound across every group (so that we can control the maximum gap between in and out of sample error across every $h \in \Hs$, within each group $k$), and
    \item We have a model that is approximately convex lexifair on our dataset $S \sim \Ps^n$, then
\end{enumerate}
then our model is also appropriately convex lexifair on the underlying distribution (with some loss in the approximation parameter). 

\begin{thm}[Generalization for Convex Lexifairness]\label{thm:convex-generalization}
Fix any distribution $\Ps$. Suppose for every $\delta > 0$, there exists $\beta(\delta)$ such that the following uniform convergence bound holds.
\[
\Pr_{S} \left[ \max_{h \in \Hs, k \in [K]} \left\vert L_k \left(h, S \right) - L_k \left(h, \Ps \right)\right\vert > \beta(\delta) \right] < \delta
\]
where $S$ is a data set sampled $i.i.d.$ from $\Ps$. We have that for every data set $S$ sampled $i.i.d.$ from $\Ps$, if a model $h$ satisfies $(\ell, \alpha)$-convex lexicographic fairness with respect to $S$, then with probability at least $1-\delta$ it also satisfies $(\ell, \alpha')$-convex lexicographic fairness with respect to $\Ps$ for $\alpha' = \alpha + 2 \ell \beta (\delta)$.
\end{thm}

\begin{proof}
Fix a distribution $\Ps$ and a data set $S$ sampled $i.i.d.$ from $\Ps$. Suppose $h$ satisfies $(\ell, \alpha)$-convex lexicographic fairness with respect to $S$. Therefore, according to our convex lexifairness definition, there exists a sequence of mappings $\epsvec = (\epsilon_1, \ldots, \epsilon_\ell)$ where $\epsilon_j \in \R^{\Hs}$, and a sequence of function classes $\{\F_{(j)}^{\epsvec} (S)\}_j$ such that
\[
\max_{1 \le j \le \ell} \left\{ \max_{h' \in \Hs} \epsilon_j (h') \right\} \le \alpha
\]
and that for all  $j \le \ell$:
\begin{equation}\label{eq:something}
\max_{ \left\{ i_1, \ldots, i_j \right\} \subseteq [K] } \sum_{r=1}^j L_{i_r} (h, S) \le \min_{g \in \F_{(j-1)}^{\epsvec}(S) } \max_{\left\{ i_1, \ldots, i_j \right\} \subseteq [K]} \sum_{r=1}^j L_{i_r} (g, S) + \epsilon_j (h) + \alpha
\end{equation}
where recall that $\F^{\epsvec}_{(0)} (S) = \Hs$ and that for all $j \in [\ell]$,
\[
    \F_{(j)}^{\epsvec} (S) = \left\{ h' \in \F_{(j-1)}^{\epsvec}(S): \max_{ \left\{ i_1, \ldots, i_j \right\} \subseteq [K] } \sum_{r=1}^j L_{i_r} (h', S) \le \min_{g \in \F_{(j-1)}^{\epsvec}(S) } \max_{\left\{ i_1, \ldots, i_j \right\} \subseteq [K]} \sum_{r=1}^j L_{i_r} (g, S) + \epsilon_j (h') \right\}
\]
Let us define a mapping $\nu_j^1 : \Hs \to \R$ such that for every $h' \in \Hs$,
\[
\nu_j^1 (h') \triangleq \max_{ \left\{ i_1, \ldots, i_j \right\} \subseteq [K] } \sum_{r=1}^j L_{i_r} (h', \Ps) - \max_{ \left\{ i_1, \ldots, i_j \right\} \subseteq [K] } \sum_{r=1}^j L_{i_r} (h', S)
\]
and let
\[
\nu_j^2 \triangleq \min_{g \in \F_{(j-1)}^{\epsvec}(S) } \max_{\left\{ i_1, \ldots, i_j \right\} \subseteq [K]} \sum_{r=1}^j L_{i_r} (g, S) - \min_{g \in \F_{(j-1)}^{\epsvec}(S) } \max_{\left\{ i_1, \ldots, i_j \right\} \subseteq [K]} \sum_{r=1}^j L_{i_r} (g, \Ps)
\]
Now define for every $h' \in \Hs$, $\tau_j (h') \triangleq \epsilon_j (h') + \nu_j^1 (h') + \nu_j^2$ and let $\F_{(j)}^{\vec{\tau}} (\Ps)$ be defined according to our convex lexifairness definition with the sequence of mappings defined by $\vec{\tau} = (\tau_1, \ldots, \tau_\ell)$. In other words, $\F_{(0)}^{\vec{\tau}} (\Ps) = \Hs$, and for all $j \in [\ell]$,
\[
    \F_{(j)}^{\vec{\tau}} (\Ps) = \left\{ h' \in \F_{(j-1)}^{\vec{\tau}}(\Ps): \max_{ \left\{ i_1, \ldots, i_j \right\} \subseteq [K] } \sum_{r=1}^j L_{i_r} (h', \Ps) \le \min_{g \in \F_{(j-1)}^{\vec{\tau}}(\Ps) } \max_{\left\{ i_1, \ldots, i_j \right\} \subseteq [K]} \sum_{r=1}^j L_{i_r} (g, \Ps) + \tau_j (h') \right\}.
\]
\begin{clm}\label{clm:something2}
For all $j$, $\F_{(j)}^{\vec{\tau}} (\Ps) = \F_{(j)}^{\epsvec} (S)$.
\end{clm}
\begin{proof}
We use induction on $j$. For $j=0$, we have $\F_{(0)}^{\vec{\tau}} (\Ps) = \F_{(0)}^{\epsvec} (S) = \Hs$. For $j \ge 1$, we have
\begin{align*}
    h' \in \F_{(j)}^{\vec{\tau}} (\Ps) &\Longleftrightarrow h' \in \F_{(j-1)}^{\vec{\tau}}(\Ps), \ \max_{ \left\{ i_1, \ldots, i_j \right\} \subseteq [K] } \sum_{r=1}^j L_{i_r} (h', \Ps) \le \min_{g \in \F_{(j-1)}^{\vec{\tau}}(\Ps) } \max_{\left\{ i_1, \ldots, i_j \right\} \subseteq [K]} \sum_{r=1}^j L_{i_r} (g, \Ps) + \tau_j (h') \\
    &\Longleftrightarrow h' \in \F_{(j-1)}^{{\epsvec}}(S), \ \max_{ \left\{ i_1, \ldots, i_j \right\} \subseteq [K] } \sum_{r=1}^j L_{i_r} (h', \Ps) \le \min_{g \in \F_{(j-1)}^{{\epsvec}}(S) } \max_{\left\{ i_1, \ldots, i_j \right\} \subseteq [K]} \sum_{r=1}^j L_{i_r} (g, \Ps) + \tau_j (h') \\
    &\Longleftrightarrow h' \in \F_{(j-1)}^{{\epsvec}}(S), \ \max_{ \left\{ i_1, \ldots, i_j \right\} \subseteq [K] } \sum_{r=1}^j L_{i_r} (h', S) \le \min_{g \in \F_{(j-1)}^{\epsvec}(S) } \max_{\left\{ i_1, \ldots, i_j \right\} \subseteq [K]} \sum_{r=1}^j L_{i_r} (g, S) + \epsilon_j (h') \\
    &\Longleftrightarrow h' \in \F_{(j)}^{{\epsvec}}(S)
\end{align*}
where the second line follows from the induction assumption ($\F_{(j-1)}^{\vec{\tau}}(\Ps) = \F_{(j-1)}^{\vec{\epsvec}}(S)$) and the third line follows from the definition of $\tau_j$. This establishes our claim.
\end{proof}
We have that for all $j \le \ell$, the model $h$ satisfies
\begin{align*}
    \max_{ \left\{ i_1, \ldots, i_j \right\} \subseteq [K] } \sum_{r=1}^j L_{i_r} (h, \Ps) &= \max_{ \left\{ i_1, \ldots, i_j \right\} \subseteq [K] } \sum_{r=1}^j L_{i_r} (h, S) + \nu_j^1 (h) \\
    &\le \min_{g \in \F_{(j-1)}^{\epsvec}(S) } \max_{\left\{ i_1, \ldots, i_j \right\} \subseteq [K]} \sum_{r=1}^j L_{i_r} (g, S) + \epsilon_j (h) + \alpha + \nu_j^1 (h) \\
    &= \min_{g \in \F_{(j-1)}^{\epsvec}(S) } \max_{\left\{ i_1, \ldots, i_j \right\} \subseteq [K]} \sum_{r=1}^j L_{i_r} (g, \Ps) + \nu_j^2 + \epsilon_j (h) + \alpha + \nu_j^1 (h) \\
    &= \min_{g \in \F_{(j-1)}^{\vec{\tau}}(\Ps) } \max_{\left\{ i_1, \ldots, i_j \right\} \subseteq [K]} \sum_{r=1}^j L_{i_r} (g, \Ps) + \tau_j (h) + \alpha
\end{align*}
where the first inequality follows from Equation~\eqref{eq:something}. The third line follows from the definition of $\nu_j^2$. The last equality follows from Claim~\ref{clm:something2} and the fact that $\tau_j (h) = \epsilon_j (h) + \nu_j^1 (h) + \nu_j^2$. The proof is complete by the uniform convergence bound provided in the theorem statement. With probability at least $1-\delta$ over the random draws of the data set $S$, we have $\max_{h' \in \Hs} \vert \nu_j^1 (h') \vert \le j\beta(\delta)$ and $\vert \nu_j^2 \vert \le j\beta(\delta)$, and hence for all $j \le \ell$,
\begin{align*}
\Vert \tau \Vert_\infty &= \max_{1 \le j \le \ell} \left\{ \max_{h' \in \Hs} \tau_j (h') \right\} \\
&\le \max_{1 \le j \le \ell} \left\{ \max_{h' \in \Hs} \epsilon_j (h') \right\} + \max_{1 \le j \le \ell} \left\{ \max_{h' \in \Hs} \vert \nu_j^1 (h') \vert + \vert \nu_j^2 \vert \right\}\\
&\le \alpha + 2 l \beta(\delta)
\end{align*}
%where the equality follows from Claim~\ref{clm:something2} and the first inequality follows from the definition of $\tau_j$.
\end{proof}

We can now instantiate the above theorem in a classification setting in which we  have VC-type convergence bounds. A corollary that we get by applying standard uniform convergence bounds for finite VC classes (See Appendix \ref{sec:VC}) is the following: 
\begin{cor}[Generalization for Convex Lexifairness: Classification Setting]
Suppose $\Hs$ is a class of binary classifiers with VC dimension $d_\Hs$ and let $L_z( p ) = \E_{h \sim p} \left[ L_z (h )\right]$ for any randomized model $p \in \Delta \Hs$ where $L_z (h) = \1 \left\{ h(x) \neq y \right\}$ is the zero-one loss. We have that for every $\Ps$, every data set $S \equiv \{ G_k \}_k$ of size $n$ sampled $i.i.d.$ from $\Ps$, if a model $p \in \Delta \Hs$ satisfies $(\ell, \alpha)$-convex lexicographic fairness with respect to $S$, then with probability at least $1-\delta$ it also satisfies $(\ell, 2 \alpha)$-convex lexicographic fairness with respect to $\Ps$ provided that
\[
\min_{1 \le k \le K} \left\vert G_k \right\vert = \Omega \left( \frac{l^2 \left( d_\Hs \log \left( n \right) + \log \left( K / \delta \right) \right)}{\alpha^2} \right)
\]
\end{cor}

We have here proven a generalization theorem for convex lexifairness (Definition~\ref{def:convexlexifair}) which is the definition that our algorithms satisfy. We also prove a generalization theorem for lexifairness (Definition~\ref{def:lexifair}) in Appendix~\ref{sec:lexifair-generalization}.

\subsection*{Acknowledgements}
Supported in part by the Warren Center for Network and Data Sciences,  NSF grant CCF-1763307 and the Simons Collaboration on the Theory of Algorithmic Fairness.

\bibliographystyle{plainnat}
\bibliography{main}

\appendix
\input{appendix}
\end{document}

%% file: abstract.tex
We extend the notion of minimax fairness in supervised learning problems to its natural conclusion: \emph{lexicographic} minimax fairness (or
\emph{lexifairness} for short). Informally, given a collection of demographic groups of interest, minimax fairness asks that the error of the group with the \emph{highest} error be minimized. Lexifairness goes further and asks that amongst all minimax fair solutions, the error of the group with the second highest error should be minimized, and amongst all of \emph{those} solutions, the error of the group with the third highest error should be minimized, and so on. Despite its naturalness, correctly defining lexifairness is considerably more subtle than minimax fairness, because of inherent sensitivity to approximation error. We give a notion of approximate lexifairness that avoids this issue, and then derive oracle-efficient algorithms for finding approximately lexifair solutions in a very general setting. When the underlying empirical risk minimization problem absent fairness constraints is convex (as it is, for example, with linear and logistic regression), our algorithms are provably efficient even in the worst case. Finally, we show generalization bounds---approximate lexifairness on the training sample implies approximate lexifairness on the true distribution with high probability. Our ability to prove generalization bounds depends on our choosing definitions that avoid the instability of naive definitions. 

%% file: relatedWork.tex
There are many notions of group or statistical fairness that are studied in the fair machine learning literature, which are generally concerned with \emph{equalizing} various measures of error across protected groups; see e.g. \cite{berk2018fairness,mitchell2021algorithmic} for surveys of many such metrics. 

Minimax solutions are a classical approach to fairness that have been used in many contexts including scheduling, fair division, and clustering (see e.g. \cite{scheduling, fairdivision, samadi2018price, chen, cotter}). A number of these works employ techniques for solving two-player zero-sum games as part of their algorithmic solution \cite{cotter, chen}. This is the same general algorithmic framework that we use. More recently, minimax group error has been proposed as a fairness solution concept for classification problems in machine learning \cite{minimax1,minimax2,lahoti2020fairness}. These works generally do not specify how to choose between multiple minimax solutions, with the exception of \cite{minimax2}, which gives algorithms for choosing the solution with smallest overall classification error subject to the minimax constraint.

Lexicographic minimax fairness has been studied in the fair division literature for tasks such as quota allocation in mobile networks, load balancing, and network design \cite{7590424, 10.1145/3011282, Nace2006ATO, adHoc, 4177725, 4346554, 4469905, RePEc:hin:jnljam:340913, ogryczak}. As far as we know, we are the first to study lexicographic fairness in a learning context in which the quantities of interest must be \emph{estimated}, and hence the first to identify the sensitivity issues that arise when defining \emph{approximate} notions of lexicographic fairness.

An alternative approach to learning one classifier for all groups is to learn \emph{decoupled classifiers} \cite{decoupled1,decoupled2}, i.e. a separate classifier for each group. The decoupling of error rates across all groups eliminates tradeoffs between groups, and hence results in classifiers that are lexicographically fair (within the class of decoupled classifiers). But there are at least three important reasons one might want to learn a single classifier (the approach we take) rather than a separate classifier for each group. The first is that learning separate classifiers for each group requires that the groups be \emph{disjoint}, which is not needed in our approach. For example, we could divide the population into groups according to race, gender, and age---despite the fact that individuals will fall into multiple groups simultaneously. In other words, our algorithms can be used to obtain \emph{subgroup} or \emph{intersectional} fairness \cite{kearns2018preventing,kearns2019empirical,hebert2018multicalibration,kim2019multiaccuracy,jung2020moment,gupta2021online}. Second, learning separate classifiers for each group requires that protected group membership be used explicitly at classification time, which can be undesirable or illegal in important applications. Finally, learning a single classifier allows for the possibility of transfer learning, whereby a small sample from some group can be partially made up for by larger quantities of data from other (nevertheless related) groups. 

%Fair Division/lexifair papers
%        Lexicographical order Max-Min Fair Source Quota Allocation in Mobile Delay-Tolerant Networks
%        A Tutorial on Max-Min Fairness and its Applications to Routing, Load-Balancing and Network Design
%        Lexicographic Max-Min Fairness in a Wireless Ad Hoc Network with Random Access
%        Lexicographic Max-Min Fair Rate Allocation in Random Access Wireless Networks
%        A Unified Framework for Max-Min and Min-Max Fairness with Applications
%        Upward Max Min Fairness
%        Centralized and Distributed Algorithms for Routingand Weighted Max-Min Fair Bandwidth Allocation
%        Fair Optimization and Networks: Models, Algorithms, and Applications
%        Lexicographic Max-Min Optimization for Efficient and Fair Bandwidth Allocati

%% file: appendix.tex
\section{Proofs From Section \ref{sec:model}}
\label{app:clm}

\convexity*

\begin{proof}[Proof of Claim~\ref{clm:convexity}]
Fix any $\epsvec$ such that $\epsilon_j$ is a concave mapping of $\Hs$. We use induction on $j$ to prove this claim. The base case $j=0$ follows from the assumption that $\Hs$ is convex. Now suppose $\F_{(j-1)}^{\epsvec}$ is convex for some $j \ge 1$. Let $h_1, h_2 \in \F_{(j)}^{\epsvec}$ and $\alpha \in (0,1)$. We want to show that $f:=\alpha h_1 + (1-\alpha)h_2 \in \F_{(j)}^{\epsvec}$. First note that $f \in \F_{(j-1)}^{\epsvec}$ because $h_1, h_2 \in \F_{(j)}^{\epsvec} \subseteq \F_{(j-1)}^{\epsvec}$. We also have that
    \begin{align*}
        \max_{ \left\{ i_1, \ldots, i_j \right\} \subseteq [K] } \sum_{r=1}^j L_{i_r} (f) &\le \max_{ \left\{ i_1, \ldots, i_j \right\} \subseteq [K]} \sum_{r=1}^j \left\{ \alpha L_{i_r} (h_1) + (1-\alpha) L_{i_r} (h_2) \right\} \\
        &\le \alpha \cdot \max_{ \left\{ i_1, \ldots, i_j \right\} \subseteq [K]} \sum_{r=1}^j L_{i_r} (h_1) + (1-\alpha) \cdot \max_{ \left\{ i_1, \ldots, i_j \right\} \subseteq [K] } \sum_{r=1}^j L_{i_r} (h_2) \\
        &\le \min_{g \in \F_{(j-1)}^{\epsvec} } \max_{\left\{ i_1, \ldots, i_j \right\} \subseteq [K]} \sum_{r=1}^j L_{i_r} (g) + \alpha \epsilon_j (h_1) + (1-\alpha) \epsilon_j (h_2) \\
        &\le \min_{g \in \F_{(j-1)}^{\epsvec} } \max_{\left\{ i_1, \ldots, i_j \right\} \subseteq [K]} \sum_{r=1}^j L_{i_r} (g) + \epsilon_j (f) 
    \end{align*}
    %\emily{Do we know that the loss function $L$ is linear? What if we are using MSE, for example? Just getting stuck on the first equality in the above expression}\saeed{Good point. We can replace that with $\le$ assuming $L$ is convex.}
    
\noindent as desired, where the first inequality follows by convexity of $L_z$, the third inequality follows because $h_1, h_2 \in \F_{(j)}^{\epsvec}$, and the last one follows by conacvity of $\epsilon_j$. This establishes our claim.
    %\ar{I think the above -should- be talking about distributions over classifiers, in which case we don't care if $L$ is convex in the model parameters or not --- since for a distribution, the loss is defined to be the expected loss, which is linear.}
\end{proof}

% \begin{clm3}[Relationship between $\F_{(j)}^{\epsvec}$ and $\Hs_{(j)}^{\epsvec}$ when $\epsvec = \vec{0}$]
%     For all $j$, $\F_{(j)}^{\vec{0}} = \Hs_{(j)}^{\vec{0}}$.
% \end{clm3}

\relationship*

\begin{proof}[Proof of Claim~\ref{clm:relationship}]
    Let $\epsvec = \vec{0}$. Before we prove the claim, note that we can use the fact that $\bar{h}$ is an ordering of the losses to omit the maximization terms in the definition of (convex approximate) $\F_{(j)}^{\epsvec}$ and rewrite it for $j \ge 1$ as follows: 
    \[
    \F_{(j)}^{\epsvec} = \left\{ h \in \F_{(j-1)}^{\epsvec}: \sum_{r=1}^j L_{\bar{h}(r)} (h) \le \min_{g \in \F_{(j-1)}^{\epsvec} } \sum_{r=1}^j L_{\bar{g}(r)} (g) + \epsilon_j \right\}
    \] 
    This formulation of $\F_{(j)}^{\epsvec}$ helps us establish the claim that $\Hs_{(j)}^{\vec{0}} = \F_{(j)}^{\vec{0}}$. Now we use induction on $j$ to prove the claim. Note that $\Hs_{(0)}^{\vec{0}} = \F_{(0)}^{\vec{0}} = \Hs$ and $\Hs_{(1)}^{\vec{0}} = \F_{(1)}^{\vec{0}}$ trivially hold. Now suppose $\Hs_{(j-1)}^{\vec{0}} = \F_{(j-1)}^{\vec{0}}$ for some $j \ge 2$. We will show that $\F_{(j)}^{\vec{0}} \subseteq \Hs_{(j)}^{\vec{0}}$ and $\Hs_{(j)}^{\vec{0}} \subseteq \F_{(j)}^{\vec{0}}$. Let $f \in \F_{(j)}^{\vec{0}}$. First, we have that $f \in \Hs_{(j-1)}^{\vec{0}} = \F_{(j-1)}^{\vec{0}}$. Second,
    \begin{align*}
        L_{\bar{f}(j)} (f) &= \sum_{r=1}^j L_{\bar{f}(r)} (f) - \sum_{r=1}^{j-1} L_{\bar{f}(r)} (f) \\
        &= \min_{g \in \F_{(j-1)}^{\vec{0}} } \sum_{r=1}^j L_{\bar{g}(r)} (g) - \sum_{r=1}^{j-1} L_{\bar{f}(r)} (f) \\
        &= \min_{g \in \F_{(j-1)}^{\vec{0}} } \left\{ L_{\bar{g}(j)} (g) + \sum_{r=1}^{j-1} L_{\bar{g}(r)} (g) \right\} - \sum_{r=1}^{j-1} L_{\bar{f}(r)} (f) \\
        &= \min_{g \in \F_{(j-1)}^{\vec{0}} } L_{\bar{g}(j)} (g) \\
        &= \min_{g \in \Hs_{(j-1)}^{\vec{0}} } L_{\bar{g}(j)} (g)
    \end{align*}
    implying that $f \in \Hs_{(j)}^{\vec{0}}$. Note that the second equation follows because $f \in \F_{(j)}^{\vec{0}}$. The fourth one follows because for all $g \in \F_{(j-1)}^{\vec{0}}$, $\sum_{r=1}^{j-1} L_{\bar{g}(r)} (g) = \sum_{r=1}^{j-1} L_{\bar{f}(r)} (f)$. The last one follows because $\Hs_{(j-1)}^{\vec{0}} = \F_{(j-1)}^{\vec{0}}$ by induction assumption. So far we have showed that $\F_{(j)}^{\vec{0}} \subseteq \Hs_{(j)}^{\vec{0}}$. It remains to show that $\Hs_{(j)}^{\vec{0}} \subseteq \F_{(j)}^{\vec{0}}$ too. Suppose $f \in \Hs_{(j)}^{\vec{0}}$. First, note that $f \in \F_{(j-1)}^{\vec{0}} = \Hs_{(j-1)}^{\vec{0}}$. We also have that
    \begin{align*}
        \sum_{r=1}^j L_{\bar{f}(r)} (f) &= \sum_{r=1}^j \min_{g \in \Hs_{(r-1)}^0 } L_{\bar{g}(r)} (g) \\
        &\le \sum_{r=1}^j \min_{g \in \Hs_{(j-1)}^{\vec{0}} } L_{\bar{g}(r)} (g) \\
        &\le \min_{g \in \Hs_{(j-1)}^{\vec{0}} } \sum_{r=1}^j L_{\bar{g}(r)} (g) \\
        &= \min_{g \in \F_{(j-1)}^{\vec{0}} } \sum_{r=1}^j L_{\bar{g}(r)} (g)
    \end{align*}
    implying that $f \in \F_{(j)}^{\vec{0}}$. Here, the first equation and the first inequality follow because $f \in \Hs_{(j)}^{\vec{0}}$ and that $\Hs_{(j)}^{\vec{0}} \subseteq \Hs_{(j-1)}^{\vec{0}} \subseteq \ldots \subseteq \Hs_{(0)}^{\vec{0}}$. The last equation follows by induction assumption that $\Hs_{(j-1)}^{\vec{0}} = \F_{(j-1)}^{\vec{0}}$. So we have showed that $\Hs_{(j)}^{\vec{0}} \subseteq \F_{(j)}^{\vec{0}}$ and this completes the proof.
\end{proof}

\section{Proofs from Section \ref{sec:finding}}
\label{app:finding}
\thmdunno*

\begin{proof}[Proof of Theorem~\ref{thm:dunno}]
We first show the following Lemma:
\begin{lemma}\label{lem:aux}
Let $x_{+} \triangleq \max (x,0)$ for $x \in \R$. We have that the strategies $((\hat{h}, \hat{\eta}_j), \hat{\lambda})$ satisfy the following:
\[
\sum_{r=1}^{j} \sum_{\{i_1,\ldots,i_r\} \subseteq [K]} \hat{\lambda}_{\{i_1,i_2, \ldots,i_r\}} \cdot \left( L_{i_1} (\hat{h}) + \ldots + L_{i_r} (\hat{h}) - \hat{\eta}_r \right) \ge B \max_{\overset{1 \le r \le j}{\{i_1,\ldots,i_r\} \subseteq [K]}} \left( L_{i_1} (\hat{h}) + \ldots + L_{i_r} (\hat{h}) - \hat{\eta}_r \right)_{+} - \nu
\]
\end{lemma}
\begin{proof}[Proof of Lemma~\ref{lem:aux}]
Let $\lambda \in \Lambda_j$ be the best response of the Auditor to $(\hat{h}, \hat{\eta}_j)$: $\lambda = \lambda_\text{best} (\hat{h}, \hat{\eta}_j)$. Then we have by the $\nu$-approximate equilibrium guarantee that
\[
\mathcal{L}_j \left( (\hat{h}, \hat{\eta}_j), \hat{\lambda} \right) \ge \mathcal{L}_j \left( (\hat{h}, \hat{\eta}_j), \lambda \right) - \nu
\]
The proof is complete by expanding the Lagrangian terms in the above inequality. We have
\[
\mathcal{L}_j \left( (\hat{h}, \hat{\eta}_j), \hat{\lambda} \right) = \hat{\eta}_j + \sum_{r=1}^j \sum_{\{i_1, \ldots, i_r\} \subseteq [K]} \hat{\lambda}_{\{i_1, i_2, \ldots, i_r\}} \cdot \left( L_{i_1} (\hat{h}) + \ldots + L_{i_r} (\hat{h}) - \hat{\eta}_r \right)
\]
and
\begin{align*}
\mathcal{L}_j \left( (\hat{h}, \hat{\eta}_j), \lambda \right) &= \hat{\eta}_j + \sum_{r=1}^j \sum_{\{i_1, \ldots, i_r\} \subseteq [K]} \lambda_{\{i_1, i_2, \ldots, i_r\}} \cdot \left( L_{i_1} (\hat{h}) + \ldots + L_{i_r} (\hat{h}) - \hat{\eta}_r \right) \\
&= \hat{\eta}_j + B \max_{\overset{1 \le r \le j}{\{i_1,\ldots,i_r\} \subseteq [K]}} \left( L_{i_1} (\hat{h}) + \ldots + L_{i_r} (\hat{h}) - \hat{\eta}_r \right)_{+}
\end{align*}
where the second equation follows by the definition of $\lambda$.
\end{proof}
With this Lemma in hand we can prove the theorem. Let $(h,\eta_j)$ be any feasible solution to the optimization problem~\eqref{eq:opt-problem}. We have that
\[
\mathcal{L}_j \left( (h, \eta_j), \hat{\lambda} \right) \le \eta_j
\]
because all of the components of $\hat{\lambda} \in \Lambda_j$ are nonnegative and that the constraints of \eqref{eq:opt-problem} are all satisfied by $(h,\eta_j)$. We also have by the $\nu$-approximate equilibrium guarantee that
\begin{equation}\label{eq:upper}
\mathcal{L}_j \left( (\hat{h}, \hat{\eta}_j), \hat{\lambda} \right) \le \mathcal{L}_j \left( (h, \eta_j), \hat{\lambda} \right) + \nu \le \eta_j + \nu
\end{equation}
But Lemma~\ref{lem:aux} implies the following lower bound on $\mathcal{L}_j \left( (\hat{h}, \hat{\eta}_j), \hat{\lambda} \right)$:
\begin{align}\label{eq:lower}
\begin{split}
\mathcal{L}_j \left( (\hat{h}, \hat{\eta}_j), \hat{\lambda} \right) &= \hat{\eta}_j + \sum_{r=1}^{j} \sum_{\{i_1,\ldots,i_r\} \subseteq [K]} \hat{\lambda}_{\{i_1,i_2, \ldots,i_r\}} \cdot \left( L_{i_1} (\hat{h}) + \ldots + L_{i_r} (\hat{h}) - \hat{\eta}_r \right) \\
&\ge \hat{\eta}_j + B \max_{\overset{1 \le r \le j}{\{i_1,\ldots,i_r\} \subseteq [K]}} \left( L_{i_1} (\hat{h}) + \ldots + L_{i_r} (\hat{h}) - \hat{\eta}_r \right)_{+} - \nu \\
&\ge \hat{\eta}_j - \nu
\end{split}
\end{align}
Combining Equations~\eqref{eq:upper} and \eqref{eq:lower} implies
$
\hat{\eta}_j \le \eta_j + 2\nu
$,
and since this condition holds for every feasible $(h, \eta_j)$, we get that
\[
\hat{\eta}_j \le OPT_j \left( \eta_1, \ldots, \eta_{j-1} \right) + 2\nu
\]
which proves the first part of the theorem. Once again using Equations~\eqref{eq:upper} and \eqref{eq:lower},
\begin{align*}
\max_{\overset{1 \le r \le j}{\{i_1,\ldots,i_r\} \subseteq [K]}} \left( L_{i_1} (\hat{h}) + \ldots + L_{i_r} (\hat{h}) - \hat{\eta}_r \right) &\le \max_{\overset{1 \le r \le j}{\{i_1,\ldots,i_r\} \subseteq [K]}} \left( L_{i_1} (\hat{h}) + \ldots + L_{i_r} (\hat{h}) - \hat{\eta}_r \right)_{+} \\
&\le \frac{\eta_j - \hat{\eta}_j + 2\nu}{B} \\
&\le \frac{jL_M + 2\nu}{B}
\end{align*}
where we use the fact that $\eta_j, \hat{\eta}_j \in [0, jL_M]$. In other words, for all $r \le j$ we have the following guarantee:
\[
\max_{\left\{ i_1, \ldots, i_r \right\} \subseteq [K]} \sum_{s=1}^r L_{i_r} (\hat{h}) \le \hat{\eta}_{r} + \frac{jL_M + 2 \nu}{B}.
\]
\end{proof}

\section{Proofs from Section \ref{sec:reg}}
\label{app:reg}
\regthm*

\begin{proof}
We will show that for every round $j$, the model $\hat{\theta}_j$ computed by our algorithm is $(j,\alpha)$-convex lexicographic fair, and as a consequence, the very last model ($\hat{\theta}_\ell$) is $(\ell,\alpha)$-convex lexicographic fair. Fix any round $j \le \ell$. Let $(\theta^t, \eta_j^t, \lambda^t)_{t=1}^T$ be the sequence of plays in the no-regret dynamics of Algorithm~\ref{alg:nr-regression} in this round. First, note that by the decomposition given in Equation~\eqref{eq:decomposition}, we have
\begin{align*}
&\sum_{t=1}^T \mathcal{L}_j \left( (\theta^t, \eta_j^t), \lambda^t \right) - \min_{\theta \in \Theta, \eta_j \in [0, j\cdot L_M]} \sum_{t=1}^T \mathcal{L}_j \left( (\theta, \eta_j), \lambda^t \right) \\ 
&= \left\{ \sum_{t=1}^T \mathcal{L}_j^1 \left( \theta^t, \lambda^t \right) - \min_{\theta \in \Theta} \sum_{t=1}^T \mathcal{L}_j^1 \left( \theta, \lambda^t \right) \right\} + \left\{ \sum_{t=1}^T \mathcal{L}_j^2 \left( \eta_j^t, \lambda^t \right) - \min_{\eta_j \in [0, j\cdot L_M]} \sum_{t=1}^T \mathcal{L}_j^2 \left( \eta_j, \lambda^t \right) \right\}
\end{align*}
In other words, we can decompose the regret of the Learner into two terms: one is the regret of gradient descent plays corresponding to $\theta$, and the other one is the corresponding regret of gradient descent plays for $\eta_j$. Note that by Equations~\eqref{eq:grad-model} and \eqref{eq:grad-eta} we have the following bounds on the norm of gradients for the Learner. We also use the fact that when the Auditor is best responding, $w_r (\lambda^t)$ can be simplified as in Fact~\ref{fact:efficient}.
\[
\left\Vert \nabla_\theta \mathcal{L}_j \left( (\theta, \eta_j), \lambda^t \right) \right\Vert_2 \le \sum_{r=1}^K \left\vert w_r (\lambda^t) \right\vert \cdot \left\Vert \nabla_\theta L_r (\theta) \right\Vert_2  \le j B G
\]
\[
\left\Vert \nabla_{\eta_j} \mathcal{L}_j \left( (\theta, \eta_j), \lambda^t \right) \right\Vert_2 = \left\vert 1 - \sum_{\{i_1, \ldots, i_j\} \subseteq [K]} \lambda^t_{\{i_1, i_2, \ldots, i_j\}} \right\vert \le 1 + B
\]
Now letting $\eta = \frac{D}{jBG\sqrt{T}}$ and $\eta' = \frac{jL_M}{(1+B)\sqrt{T}}$ in Algorithm~\ref{alg:nr-regression} and using the regret bound of Online Projected Gradient Desccent (Theorem~\ref{thm:gdregret}), we have
\[
\sum_{t=1}^T \mathcal{L}_j^1 \left( \theta^t, \lambda^t \right) - \min_{\theta \in \Theta} \sum_{t=1}^T \mathcal{L}_j^1 \left( \theta, \lambda^t \right) \le jBGD\sqrt{T}
\]
\[
\sum_{t=1}^T \mathcal{L}_j^2 \left( \eta_j^t, \lambda^t \right) - \min_{\eta_j \in [0, j\cdot L_M]} \sum_{t=1}^T \mathcal{L}_j^2 \left( \eta_j, \lambda^t \right) \le j (B+1) L_M \sqrt{T}
\]
and therefore the regret of the Learner can be bounded by
\[
\sum_{t=1}^T \mathcal{L}_j \left( (\theta^t, \eta_j^t), \lambda^t \right) - \min_{\theta \in \Theta, \eta_j \in [0, j\cdot L_M]} \sum_{t=1}^T \mathcal{L}_j \left( (\theta, \eta_j), \lambda^t \right) \le j (GD + L_M) (B+1) \sqrt{T} :=\nu_j T
\]

Let $\nu_j \triangleq j (GD + L_M) (B+1) / \sqrt{T}$. Now using the guarantees of the no-regret dynamics (Theorem~\ref{thm:noregret}), the average play of the players $(\hat{\theta}, \hat{\eta}_j, \hat{\lambda})$ forms a $\nu_j$-approximate equilibrium of the game in the sense that
\[
\mathcal{L}_j \left( (\hat{\theta}, \hat{\eta}_j), \hat{\lambda} \right) \le \min_{\theta \in \Theta, \eta_j \in [0, j\cdot L_M]} \mathcal{L}_j \left( (\theta, \eta_j), \hat{\lambda} \right) + \nu_j, \quad \mathcal{L}_j \left( (\hat{\theta}, \hat{\eta}_j), \hat{\lambda} \right) \ge \max_{\lambda \in \Lambda_j} \mathcal{L}_j \left( (\hat{\theta}, \hat{\eta}_j), \lambda \right) - \nu_j
\]
Finally, using Theorem~\ref{thm:dunno} we can turn these into the following guarantees. First,
\begin{equation}\label{eq:firstbound}
\hat{\eta}_j \le OPT_j \left( \hat{\eta}_1, \ldots, \hat{\eta}_{j-1} \right) + 2\nu_j
\end{equation}
and second, for all $r \le j$,
\begin{equation}\label{eq:secondbound}
\max_{\left\{ i_1, \ldots, i_r \right\} \subseteq [K]} \sum_{s=1}^r L_{i_r} (\hat{\theta}_j) \le \hat{\eta}_{r} + \frac{jL_M + 2 \nu_j}{B}
\end{equation}

Define $\epsilon_r \triangleq \hat{\eta}_r - OPT_r \left( \hat{\eta}_1, \ldots, \hat{\eta}_{r-1} \right)$ for all $r \le j$ ($\epsilon$'s here are basically \emph{constant} mappings in $\R^\Hs$). We immediately have from Equation~\eqref{eq:firstbound} that: $\epsilon_r \le 2 \nu_r$, for all $r \le j$. Now let $\epsvec = (\epsilon_1, \ldots, \epsilon_j)$, and let $\F_{(0)}^{\epsvec} = \Theta$ be the initial model class. Note that according to Definition~\ref{def:convexlexifair} and given the defined $\epsvec$, we have for every $r \le j$,
\[
\min_{\theta \in \F_{(r-1)}^{\epsvec} } \max_{\left\{ i_1, \ldots, i_r \right\} \subseteq [K]} \sum_{s=1}^r L_{i_r} (\theta) \equiv OPT_r \left( \hat{\eta}_1, \ldots, \hat{\eta}_{r-1} \right)
\]
And therefore, by Equation~\eqref{eq:secondbound}, for all $r \le j$:
\begin{align*}
\max_{\left\{ i_1, \ldots, i_r \right\} \subseteq [K]} \sum_{s=1}^r L_{i_r} (\hat{\theta}_j) &\le \hat{\eta}_{r} + \frac{j L_M + 2 \nu_r}{B} \\
&= OPT_r \left( \hat{\eta}_1, \ldots, \hat{\eta}_{r-1} \right) + \epsilon_r + \frac{j L_M + 2 \nu_r}{B} \\
&= \min_{g \in \F_{(r-1)}^{\epsvec} } \max_{\left\{ i_1, \ldots, i_r \right\} \subseteq [k]} \sum_{s=1}^r L_{i_r} (g) + \epsilon_r + \frac{j L_M + 2 \nu_r}{B}
\end{align*}
which completes the proof by the choice of $\nu_r = \frac{ \alpha}{2}$ for all $r \le j$ (to guarantee that $\Vert \epsvec \Vert_\infty \le \alpha$), and $B= \frac{\alpha + j L_M}{\alpha}$. Note that this setting of parameters, together with $\nu_j = j (GD + L_M) (B+1) / \sqrt{T}$, implies that
\[
T = \frac{4 j^2 (GD + L_M)^2 (2 \alpha + j L_M)^2}{\alpha^4}
\]
\end{proof}

\section{Proofs from Section \ref{sec:class}}
\label{app:class}
\mainthm*

\begin{proof}
We will show that for any round $j$, the model $\hat{\theta}_j$ computed by Algorithm~\ref{alg:fair-classification} is $(j,\alpha)$-convex lexifair. Fix any round $j$. We prove the claim in the following steps.
\begin{enumerate}
    \item \emph{The Learner's regret}: First, we invoke Theorem~\ref{thm:ftplregret} (the regret guarantee of the FTPL algorithm) and the fact that the payoff function of the game (the Lagrangian $\mathcal{L}_j$) is separable for the Learner, to write the expected regret of the distributions maintained by the Learner's FTPL algorithm, i.e. ($p^1, \ldots, p^T$) and $(D^1, \ldots, D^T)$, as follows:
    \begin{align*}
    R_L &\triangleq \frac{1}{T}\sum_{t=1}^T \E_{h \sim p^t, \eta_j \sim D^t} \left[ \mathcal{L}_j \left( (h,\eta_j), \lambda^t \right) \right] - \frac{1}{T} \min_{h \in \Hs (S), \eta_j \in \{0,j\}} \sum_{t=1}^T \mathcal{L}_j \left( (h,\eta_j), \lambda^t \right) \\
    &\le \frac{2 n^{3/2} B}{n_{min} \sqrt{T}} + \frac{2 (1+B)}{\sqrt{T}}
    \end{align*}
    Note that we have $\vert c_i (\lambda) \vert \le B/n_{min}$ where $n_{min} \triangleq \min_{1 \le k \le K} n_k$ is the smallest group size, and $\vert c (\lambda) \vert \le 1 + B$. Also, note that the above regret guarantee holds for the following choices of the learning rates $\eta$ and $\eta'$ for the FTPL algorithm.
    \[
    \eta = \frac{n_{min}}{B} \sqrt{\frac{1}{nT}} \quad \text{and} \quad \eta' = \frac{1}{1+B} \sqrt{\frac{1}{T}}
    \]
    \item \emph{The Auditor's regret}: At each round $t$ of the game, the Auditor is best responding to $(\hat{p}^t, D^t)$ of the Learner where $\hat{p}^t$ is the empirical distribution of the $m$ sampled hypotheses from $p^t$. Therefore, the Auditor is \emph{approximately} best responding to $(p^t, D^t)$, and consequently, it is accumulating some regret over the course of the algorithm. We show the Auditor's regret is small using a uniform convergence bound which holds with high probability over the random draws of the $m$ sampled hypotheses from $p^t$. We first remind the reader of Chernoff-Hoeffding's concentration bound in Lemma~\ref{lem:chernoff}, and then move on to argue about the uniform convergence bound in Lemma~\ref{lem:concentration}. Finally, the Auditor's regret is computed in Lemma~\ref{lem:auditorregret} using the result of Lemma~\ref{lem:concentration}.
    
\begin{lemma}[Chernoff-Hoeffding's Concentration]\label{lem:chernoff}
Let $X_1, X_2, \ldots, X_n$ be $i.i.d.$ draws from a distribution with mean $\mu$ and support $[a,b] \subset \R$. We have that with probability at least $1-\delta$ over the random draws,
\[
\left\vert \frac{1}{n} \sum_{i=1}^n X_i - \mu \right\vert \le (b-a) \sqrt{\frac{\log \left( 2 / \delta \right)}{2n}}
\].
\end{lemma}

\begin{lemma}\label{lem:concentration}
Let $p$ be any distribution over $\Hs(S)$, and let $\hat{p}$ be the empirical distribution of $m$ $i.i.d.$ draws from $p$. We have that for any $\delta$, with probability at least $1-\delta$ over the random draws from $p$, for any distribution $D$ over the interval $[0,j]$,
\[
\max_{\lambda \in \Lambda_j} \left\vert \E_{h \sim \hat{p}, \eta_j \sim D} \left[ \mathcal{L}_j \left( (h,\eta_j), \lambda \right) \right] - \E_{h \sim p, \eta_j \sim D} \left[ \mathcal{L}_j \left( (h,\eta_j), \lambda \right) \right] \right\vert \le KB\sqrt{\frac{\log \left( 2K/\delta\right)}{2m}}.
\]
\end{lemma}
\begin{proof}[Proof of Lemma~\ref{lem:concentration}]
We have that for any $\lambda \in \Lambda_j$, by the decomposition given in Equation~\ref{eq:decomposition},
\begin{align*}
\left\vert \E_{h \sim \hat{p}, \eta_j \sim D} \left[ \mathcal{L}_j \left( (h,\eta_j), \lambda \right) \right] - \E_{h \sim p, \eta_j \sim D} \left[ \mathcal{L}_j \left( (h,\eta_j), \lambda \right) \right] \right\vert &= \left\vert \E_{h \sim \hat{p}} \left[ \mathcal{L}_j^1 \left( h, \lambda \right) \right] - \E_{h \sim p} \left[ \mathcal{L}_j^1 \left( h, \lambda \right) \right] \right\vert \\
&\le \sum_{r=1}^K \left\vert w_r (\lambda) \right\vert \cdot \left\vert \E_{h \sim \hat{p}} \left[ L_r (h)\right] - \E_{h \sim p} \left[ L_r (h)\right] \right\vert \\
&\le B \sum_{r=1}^K \left\vert \E_{h \sim \hat{p}} \left[ L_r (h)\right] - \E_{h \sim p} \left[ L_r (h)\right] \right\vert \\
&\le KB\sqrt{\frac{\log \left( 2K/\delta\right)}{2m}}
\end{align*}
where the last inequality holds with probability $1-\delta$ and follows from Lemma~\ref{lem:chernoff} and a union bound.
\end{proof}

\begin{lemma}[Auditor's Regret]\label{lem:auditorregret}
Let ($p^1, \ldots, p^T$) and $(D^1, \ldots, D^T)$ be the sequence of distributions maintained by the Learner's FTPL algorithm, and let $\left( \lambda^1, \lambda^2, \ldots, \lambda^T\right)$ be the sequence of Auditor's plays in Algorithm~\ref{alg:nr-classification}. We have that for every $\delta$, with probability at least $1-\delta/2$, the regret of the Auditor is bounded as follows:
\begin{align*}
 R_A &\triangleq \frac{1}{T} \max_{\lambda \in \Lambda_j} \sum_{t=1}^T \E_{h \sim p^t, \eta_j \sim D^t} \left[ \mathcal{L}_j \left( (h,\eta_j), \lambda \right) \right] - \frac{1}{T} \sum_{t=1}^T \E_{h \sim p^t, \eta_j \sim D^t} \left[ \mathcal{L}_j \left( (h,\eta_j), \lambda^t \right) \right] \\
 &\le KB\sqrt{\frac{2\log \left( 4KT/\delta\right)}{m}}
\end{align*}
\end{lemma}
\begin{proof}[Proof of Lemma~\ref{lem:auditorregret}]
We know that the regret of the Auditor to $(\hat{p}^1, \ldots, \hat{p}^T)$ and $(D^1, \ldots, D^T)$ is zero. In other words,
\[
\max_{\lambda \in \Lambda_j} \sum_{t=1}^T \E_{h \sim \hat{p}^t, \eta_j \sim D^t} \left[ \mathcal{L}_j \left( (h,\eta_j), \lambda \right) \right] - \sum_{t=1}^T \E_{h \sim \hat{p}^t, \eta_j \sim D^t} \left[ \mathcal{L}_j \left( (h,\eta_j), \lambda^t \right) \right] \le 0
\]
We have by Lemma~\ref{lem:concentration} that, with probability at least $1-\delta$, for any $\lambda \in \Lambda_j$,
\begin{align*}
&\left\vert \E_{h \sim \hat{p}^t, \eta_j \sim D^t} \left[ \mathcal{L}_j \left( (h,\eta_j), \lambda \right) \right] - \E_{h \sim p^t, \eta_j \sim D^t} \left[ \mathcal{L}_j \left( (h,\eta_j), \lambda \right) \right] \right\vert \le KB\sqrt{\frac{\log \left( 2K/\delta\right)}{2m}}
\end{align*}
A union bound over $t \in [T]$ completes the proof.
\end{proof}
We can now continue with the remaining steps of the proof. 
    \item \emph{Equilibrium guarantees of the average play}: Let $\bar{p} = \frac{1}{T} \sum_{t=1}^t p^t $ and $\bar{D} = \frac{1}{T} \sum_{t=1}^t D^t$ and $\bar{\lambda} = \frac{1}{T} \sum_{t=1}^t \lambda^t$ form the average play of the players. We have that $((\bar{p}, \bar{D}), \bar{\lambda})$ is a $(R_L+R_A)$-approximate equilibrium of the game by Theorem~\ref{thm:noregret}.
    \item \emph{Additional error due to sampling from $\bar{p}$}: Finally, the algorithm outputs a sparse version of $\bar{p}$: $\hat{p}$ which is the empirical distribution over $m$ $i.i.d.$ draws from $\bar{p}$. We need to show that the additional regret (let's call it $R$) due to this approximation is small. But note we can simply bound $R$ by Lemma~\ref{lem:concentration}: for any $\delta$, with probability at least $1-\delta/2$, we have
    \[
    R \le KB\sqrt{\frac{2\log \left( 4K/\delta\right)}{m}}.
    \]
    \item \emph{Final equilibrium guarantees}: We have that strategies of the players $((\hat{p}, \bar{D}), \bar{\lambda})$ form a $(R_L + R_A + R)$-approximate equilibrium of the game where,
    \[
    R_L + R_A + R \le \nu \triangleq \frac{2 n^{3/2} B}{n_{min} \sqrt{T}} + \frac{2 (1+B)}{\sqrt{T}} + KB\sqrt{\frac{2\log \left( 4KT/\delta\right)}{m}} + KB\sqrt{\frac{2\log \left( 4K/\delta\right)}{m}}.
    \]
    \item The rest of the proof is similar to the proof of Theorem~\ref{thm:lexifair-reg}. In particular, similar to the proof of Theorem~\ref{thm:lexifair-reg}, we can now apply Theorem~\ref{thm:dunno} to get turn the equilibrium guarantees into the following guarantees for our lexifair notion: with probability at least $1-\delta$,
    \begin{equation}\label{eq:firstbound2}
    \hat{\eta}_j \le OPT_j \left( \hat{\eta}_1, \ldots, \hat{\eta}_{j-1} \right) + 2\nu,
    \end{equation}
    and second, for all $r \le j$,
    \begin{equation}\label{eq:secondbound2}
    \max_{\left\{ i_1, \ldots, i_r \right\} \subseteq [K]} \sum_{s=1}^r L_{i_r} (\hat{p}_j) \le \hat{\eta}_{r} + \frac{j + 2 \nu}{B}.
    \end{equation}
    Define $\epsilon_r \triangleq \hat{\eta}_r - OPT_r \left( \hat{\eta}_1, \ldots, \hat{\eta}_{r-1} \right)$ for all $r \le j$. We immediately have from Equation~\eqref{eq:firstbound2} and a union bound that: with probability at least $1-\delta$, $\epsilon_r \le 2 \nu'$, for all $r \le j$, where
    \[
    \nu' = \frac{2 n^{3/2} B}{n_{min} \sqrt{T}} + \frac{2 (1+B)}{\sqrt{T}} + KB\sqrt{\frac{2\log \left( 4jKT/\delta\right)}{m}} + KB\sqrt{\frac{2\log \left( 4jK/\delta\right)}{m}} \ge \nu.
    \]
    Note that $\nu'$ is basically $\nu$ with $\delta$ being replaced by $\delta/j$ because of the union bound over the first $j$ rounds. Now let $\epsvec = (\epsilon_1, \ldots, \epsilon_j)$, and let $\F_{(0)}^{\epsvec} = \Delta \Hs$ be the initial model class. Note that according to Definition~\ref{def:convexlexifair} and given the defined $\epsvec$, we have for every $r \le j$,
    \[
    \min_{p \in \F_{(r-1)}^{\epsvec} } \max_{\left\{ i_1, \ldots, i_r \right\} \subseteq [K]} \sum_{s=1}^r L_{i_r} (p) \equiv OPT_r \left( \hat{\eta}_1, \ldots, \hat{\eta}_{r-1} \right)
    \]
    And therefore, by Equation~\eqref{eq:secondbound2}, for all $r \le j$:
    \begin{align*}
    \max_{\left\{ i_1, \ldots, i_r \right\} \subseteq [K]} \sum_{s=1}^r L_{i_r} (\hat{p}_j) &\le \hat{\eta}_{r} + \frac{j + 2 \nu}{B} \\
    &= OPT_r \left( \hat{\eta}_1, \ldots, \hat{\eta}_{r-1} \right) + \epsilon_r + \frac{j  + 2 \nu}{B} \\
    &= \min_{p \in \F_{(r-1)}^{\epsvec} } \max_{\left\{ i_1, \ldots, i_r \right\} \subseteq [k]} \sum_{s=1}^r L_{i_r} (p) + \epsilon_r + \frac{j + 2 \nu}{B} \\
    &\le \min_{p \in \F_{(r-1)}^{\epsvec} } \max_{\left\{ i_1, \ldots, i_r \right\} \subseteq [k]} \sum_{s=1}^r L_{i_r} (p) + \epsilon_r + \frac{j + 2 \nu'}{B}
    \end{align*}
    which completes the proof by the choice of $\nu' = \frac{\alpha}{2}$, and $B= \frac{\alpha + j}{\alpha}$. Note that this setting of parameters implies that
    \[
    m=\frac{K^2 n_{min}^2 T \log \left( 4jKT/\delta\right)}{2 n^3}, \quad T = \frac{256 \left(2 \alpha + j \right)^2 n^3}{\alpha^4 n_{min}^2}.
    \]
\end{enumerate}
\end{proof}

\section{No-Regret Learning Algorithms}
\label{app:gd+ftpl}
\subsection{Online Projected Gradient Descent}\label{subsec:gd}
Consider an online setting where a learner is playing against an adversary. The learner's action space is some Euclidean subspace $\Theta \subseteq \R^d$ which is equipped with the $\ell_2$ norm denoted by $\left\Vert \cdot \right\Vert_2$. At every round $t$ of the interaction between the learner and the adversary, the learner picks an action $\theta^t \in \Theta$ and the adversary chooses a loss function $\ell^t: \Theta \to \R_{\ge 0}$. The learner then incurs a loss of $\ell^t(\theta^t)$ at that round. Suppose the learner is using some algorithm $\mathcal{A}$ to update its actions from round to round. The goal for the learner is that the regret of $\mathcal{A}$ defined as
\[
R_\mathcal{A}(T) \triangleq \sum_{t=1}^T \ell^t (\theta^t) - \min_{\theta \in \Theta} \sum_{t=1}^T \ell^t (\theta)
\]
grows sublinearly in $T$. When $\Theta$ and the loss functions played by the adversary are convex, a standard choice of algorithm to use for the learner is \emph{Online Projected Gradient Descent} (Algorithm~\ref{alg:descent}), where in each round, the algorithm updates its action $\theta^{t+1}$ for the next round by taking a step in the opposite direction of the gradient of the loss function evaluated at the action of that round: $\nabla \ell^{t} (\theta^{t})$. The updated action is then projected onto the feasible action space $\Theta$: $\text{Proj}_{\Theta} (\theta) \triangleq \argmin_{\theta' \in \Theta} \left\Vert \theta - \theta' \right\Vert_2$. Note if the loss functions are not differentiable, we can use subgradients (which are defined given the convexity of the loss functions) instead of gradients and the guarantees will remain.

\begin{algorithm}[t]
\KwIn{learning rate $\eta$}
Initialize the learner $\theta^1 \in \Theta$\;
\For{$t=1, 2, \ldots$}{
Learner plays action $\theta^{t}$\;
Adversary plays loss function $\ell^t$\;
Learner incurs loss of $\ell^t (\theta^t)$\;
Learner updates its action:
\[
\theta^{t+1} = \text{Proj}_{\Theta} \left( \theta^{t} - \eta \nabla \ell^{t} (\theta^{t}) \right)
\]
}
\caption{Online Projected Gradient Descent}
\label{alg:descent}
\end{algorithm}

\begin{thm}[Regret for Online Projected Gradient Descent \cite{GD}]\label{thm:gdregret} 
    Suppose $\Theta \subseteq \R^d$ is convex, compact and has bounded  diameter $D$: $\sup_{\theta, \theta' \in \Theta} \left\Vert \theta - \theta' \right\Vert_2 \le D$. Suppose for all $t$, the loss functions $\ell^t$ are convex and that there exists some $G$ such that $\left\Vert \nabla \ell^t (\cdot) \right\Vert_2 \le G$. Let $\mathcal{A}$ be Algorithm~\ref{alg:descent} run with learning rate $\eta = D/(G \sqrt{T})$. We have that for every sequence of loss functions $(\ell^1, \ell^2, \ldots, \ell^T)$ played by the adversary, $R_\mathcal{A}(T) \le GD \sqrt{T}$.
\end{thm}

\subsection{Follow the Perturbed Leader}\label{subsec:ftpl}
Follow the Perturbed Leader is another no-regret learning algorithm that can sometimes be applied even when the action space of the learner is too large to run gradient descent (over an appropriately convexified space). In this case, Follow the Perturb Leader can be applied given access to an optimization oracle. 
Consider an online setting, where again a learner is playing against an adversary. Here assume the learner's action space is $A \subseteq \{0,1\}^d$. At every round $t$, the learner chooses an action $a^t \in A$ and then the adversary plays a loss vector $\ell^t \in \R^d$. The learner then incurs a loss of $\langle \ell^t, a^t \rangle$ which is the inner product if $a^t$ and $\ell^t$. Suppose the learner is using some algorithm $\mathcal{A}$ to pick its actions in every round. The goal for the learner is to ensure that the regret of $\mathcal{A}$ defined as
\[
R_\mathcal{A}(T) \triangleq \sum_{t=1}^T \langle \ell^t, a^t \rangle - \min_{a \in \mathcal{A}} \sum_{t=1}^T \langle \ell^t, a \rangle
\]
grows sublinearly in $T$.  \emph{Follow the Perturbed Leader (FTPL)} (\cite{KALAI2005291}), which is described in Algorithm~\ref{alg:ftpl}, can provide guarantees in this setting. This algorithm perturbs the cumulative loss vector with appropriately scaled noise and then picks an action that minimizes the loss. Note that to implement it, we only require the ability to solve for the $\argmin$ --- i.e. we need an optimization oracle. The guarantees of this algorithm are stated below.

\begin{algorithm}[t]
\KwIn{learning rate $\eta$}
Initialize the learner $a^1 \in A$\;
\For{t = 1,2, \ldots}{
Learner plays action $a^{t}$\;
Adversary plays loss vector $\ell^t$\;
Learner incurs loss of $\langle \ell^t, a^t \rangle$\;
Learner updates its action:
\[
a^{t+1} = \argmin_{a \in A} \left\{ \left\langle \sum_{s \le t} \ell^s , a \right\rangle + \frac{1}{\eta} \left\langle \xi^t , a \right\rangle \right\}
\]
where $\xi^t \sim Uniform \left( [0,1]^d \right)$, independent of every other randomness.
}
\caption{Follow the Perturbed Leader (FTPL)}
\label{alg:ftpl}
\end{algorithm}

\begin{thm}[Regret of FTPL \cite{KALAI2005291}]\label{thm:ftplregret}
    Suppose for all $t$, $\ell^t \in [-M,M]^d$. Let $\mathcal{A}$ be Algorithm~\ref{alg:ftpl} run with learning rate $\eta = 1/(M \sqrt{dT})$. We have that for every sequence of loss vectors $(\ell^1, \ell^2, \ldots, \ell^T)$ played by the adversary, $\E \left[ R_\mathcal{A}(T) \right] \le 2 M d^{3/2} \sqrt{T}$, where expectation is taken with respect to the randomness in $\mathcal{A}$. 
\end{thm}

\section{A Standard Uniform Convergence Theorem}
\label{sec:VC}
In this section, we state a standard uniform convergence theorem for binary classifiers (with 0/1 loss) that have bounded VC-dimension. We observe that the standard bound easily extends to randomized classifiers (i.e. distributions over classifiers in a finite VC class) because of the linearity of expectation.

\begin{thm}
\label{thm:generalization}
Fix any $\delta>0$. Let $d_\Hs$ be the VC dimension of the class $\Hs$, and let $n_1,..,n_K$ be the sample sizes of groups $k=1,...,K$ in sample $S$ drawn from distribution $\Ps$, and let $n=n_1 + \ldots + n_K$ be the total sample size. Recall that $L_k(h,S)$ denotes the error rate of $h$ on the $n_k$ samples of group $k$ in $S$, and let $L_k(h,\Ps)$ denote the expected error of $h$ with respect to $\Ps$ conditioned on group $k$. Then with probability at least $1-\delta$ over the randomness of $S$, for every randomized classifier $p \in \Delta \Hs$, and every group $k \in [K]$:
\[
|L_k(p,\Ps) - L_k(p,S)| = O\left( \sqrt{\frac{\log \left(K/\delta \right) + d_\Hs \log{n}}{\min_k n_k}}\right)
\]

%Furthermore, let $L_{\bar{h}}(D,h)$ and $L_{\bar{h}}(S,h)$ be the sorted counterparts of of the out-of-sample and in-sample group errors, respectively; then, for every $h \in H$, with probability $1- K\delta$ over the randomness of $S$:
%\[
%||L_{\bar{h}}(D,h) - %L_{\bar{h}}(S,h)||_{\infty} \leq %O\left( \max_j %\sqrt{\frac{\log{\frac{1}{\delta}} + d %\log{n_j}}{n_j}}\right)
%\]
\end{thm}

\begin{proof} Fix any group $k$.
A standard uniform convergence argument tells us that with probability $1-\delta$ over the $n_k$ samples from group $k$, for every (deterministic) $h \in \Hs$, the generalization gap is of order \cite{vapnik, kearnsVazirani}:
\[
|L_k(h,\Ps) - L_k(h,S) | \leq \epsilon_k = O\left(\sqrt{\frac{\log{(1/\delta)} + d_\Hs \log n_k}{n_k}}\right)
\]
Now consider a randomized model $p \in \Delta \Hs$, which is distribution over $\Hs$. We have that
\[
\left\vert L_k(p,\Ps) - L_k(p,S) \right\vert =  \left\vert \E_{h \sim p} [L_k(h,\Ps)] - \E_{h \sim p} [L_k(h,S)] \right\vert \le \E_{h \sim p} \left\vert L_k(h,\Ps) - L_k(h,S) \right\vert
\]
and as a consequence, for every group $k$, we have with probability $1-\delta$ that
\[
\max_{p \in \Delta \Hs }\left\vert L_k(p,\Ps) - L_k(p,S) \right\vert \le \max_{h \in \Hs }\left\vert L_k(h,\Ps) - L_k(h,S) \right\vert \le \epsilon_k
\]
The proof is complete by a union bound over $K$ groups.
\end{proof}

\section{A Generalization Theorem for Lexifairness (Definition~\ref{def:lexifair})}
\label{sec:lexifair-generalization}
In this section we prove a generalization theorem for Definition~\ref{def:lexifair}. The proof style is identical to that of Theorem~\ref{thm:convex-generalization}. We also make use of the following simple fact:
\begin{fact}\label{fact:order}
Let $a = (a_1, \ldots, a_n)$ and $b = (b_1, \ldots, b_n)$ be such that for all $i$, $|a_i - b_i| \le \epsilon$ for some $\epsilon$. Let $a_{(i)}$ and $b_{(i)}$ denote the $i$'th highest number in $a$ and $b$, respectively. We have that for all $i$, $|a_{(i)} - b_{(i)}| \le \epsilon$.
\end{fact}
\begin{thm}(Generalization for Lexifairness)
Fix any distribution $\Ps$. Suppose for every $\delta > 0$, there exists $\beta(\delta)$ such that the following uniform convergence bound holds.
\[
\Pr_{S} \left[ \max_{h \in \Hs, k \in [K]} \left\vert L_k \left(h, S \right) - L_k \left(h, \Ps \right)\right\vert > \beta(\delta) \right] < \delta
\]
where $S$ is a data set sampled $i.i.d.$ from $\Ps$. We have that for every data set $S$ sampled $i.i.d.$ from $\Ps$, if a model $h$ satisfies $(\ell, \alpha)$-lexicographic fairness with respect to $S$, then with probability at least $1-\delta$ it also satisfies $(\ell, \alpha')$-lexicographic fairness with respect to $\Ps$ for $\alpha' = \alpha + 2 \beta (\delta)$.
\end{thm}

\begin{proof}
Fix a distribution $\Ps$ and a data set $S$ sampled $i.i.d.$ from $\Ps$. Suppose $h$ satisfies $(\ell, \alpha)$-lexicographic fairness with respect to $S$. Therefore, according to our lexifairness definition, there exists a sequence of mappings $\epsvec = (\epsilon_1, \ldots, \epsilon_\ell)$ where $\epsilon_j \in \R^{\Hs}$, and a sequence of function classes $\{\Hs_{(j)}^{\epsvec} (S)\}_j$ such that
\[
\max_{1 \le j \le \ell} \left\{ \max_{h' \in \Hs} \epsilon_j (h') \right\} \le \alpha
\]
and that for all  $j \le \ell$:
\begin{equation}\label{eq:something3}
L_{\bar{h}_S(j)} (h, S) \le \min_{g \in \Hs^{\epsvec}_{(j-1)} (S)} L_{\bar{g}_S (j)} (g, S) + \epsilon_j (h) + \alpha
\end{equation}
where recall that $\Hs^{\epsvec}_{(0)} (S) = \Hs$ and that for all $j \in [\ell]$,
\[
    \Hs_{(j)}^{\epsvec} (S) = \left\{h' \in \Hs^{\epsvec}_{(j-1)} (S): L_{\bar{h'}
    _S (j)} (h', S) \le \min_{g \in \Hs^{\epsvec}_{(j-1)} (S)} L_{\bar{g}_S (j)} (g, S) + \epsilon_j (h') \right\}
\]
Let us define a mapping $\nu_j^1 : \Hs \to \R$ such that for every $h' \in \Hs$,
\[
\nu_j^1 (h') \triangleq L_{\bar{h'}_\Ps (j)} (h', \Ps) - L_{\bar{h'}_S (j)} (h', S)
\]
I.e., $\nu_j^1 (h')$ is the $j$'th highest error induced by $h'$ on the distribution $\Ps$ minus the $j$'th highest error induced by $h'$ on the sample $S$. Also define
\[
\nu_j^2 \triangleq \min_{g \in \Hs_{(j-1)}^{\epsvec}(S) } L_{\bar{g}_S (j)} (g, S) - \min_{g \in \Hs_{(j-1)}^{\epsvec}(S) } L_{\bar{g}_\Ps (j)} (g, \Ps)
\]
Now define for every $h' \in \Hs$, $\tau_j (h') \triangleq \epsilon_j (h') + \nu_j^1 (h') + \nu_j^2$ and let $\Hs_{(j)}^{\vec{\tau}} (\Ps)$ be defined according to our lexifairness definition with the sequence of mappings defined by $\vec{\tau} = (\tau_1, \ldots, \tau_\ell)$. In other words, $\Hs_{(0)}^{\vec{\tau}} (\Ps) = \Hs$, and for all $j \in [\ell]$,
\[
    \Hs^{\vec{\tau}}_{(j)} (\Ps) \triangleq \left\{h' \in \Hs^{\vec{\tau}}_{(j-1)} (\Ps): L_{\bar{h'}_\Ps(j)} (h, \Ps) \le \min_{g \in \Hs^{\vec{\tau}}_{(j-1)} (\Ps)} L_{\bar{g}_\Ps (j)} (g, \Ps) + \tau_j (h') \right\}
\]
\begin{clm}\label{clm:something4}
For all $j$, $\Hs_{(j)}^{\vec{\tau}} (\Ps) = \Hs_{(j)}^{\epsvec} (S)$.
\end{clm}
\begin{proof}
We use induction on $j$. For $j=0$, we have $\Hs_{(0)}^{\vec{\tau}} (\Ps) = \Hs_{(0)}^{\epsvec} (S) = \Hs$. For $j \ge 1$, we have
\begin{align*}
    h' \in \Hs_{(j)}^{\vec{\tau}} (\Ps) &\Longleftrightarrow h' \in \Hs_{(j-1)}^{\vec{\tau}}(\Ps), \ L_{\bar{h'}_\Ps(j)} (h, \Ps) \le \min_{g \in \Hs^{\vec{\tau}}_{(j-1)} (\Ps)} L_{\bar{g}_\Ps (j)} (g, \Ps) + \tau_j (h') \\
    &\Longleftrightarrow h' \in \Hs_{(j-1)}^{\epsvec} (S), \ L_{\bar{h'}_\Ps(j)} (h, \Ps) \le \min_{g \in \Hs^{\epsvec}_{(j-1)} (S)} L_{\bar{g}_\Ps (j)} (g, \Ps) + \tau_j (h') \\
    &\Longleftrightarrow h' \in \Hs_{(j-1)}^{\epsvec} (S), \ L_{\bar{h'}
    _S (j)} (h', S) \le \min_{g \in \Hs^{\epsvec}_{(j-1)} (S)} L_{\bar{g}_S (j)} (g, S) + \epsilon_j (h') \\
    &\Longleftrightarrow h' \in \Hs_{(j)}^{{\epsvec}}(S)
\end{align*}
where the second line follows from the induction assumption ($\Hs_{(j-1)}^{\vec{\tau}}(\Ps) = \Hs_{(j-1)}^{\vec{\epsvec}}(S)$) and the third line follows from the definition of $\tau_j$. This establishes our claim.
\end{proof}
We have that for all $j \le \ell$, the model $h$ satisfies
\begin{align*}
    L_{\bar{h}_\Ps (j)} (h, \Ps) &= L_{\bar{h}_S(j)} (h, S) + \nu_j^1 (h) \\
    &\le \min_{g \in \Hs^{\epsvec}_{(j-1)} (S)} L_{\bar{g}_S (j)} (g, S) + \epsilon_j (h) + \alpha + \nu_j^1 (h) \\
    &= \min_{g \in \Hs^{\epsvec}_{(j-1)} (S)} L_{\bar{g}_\Ps (j)} (g, \Ps) + \nu_j^2 + \epsilon_j (h) + \alpha + \nu_j^1 (h) \\
    &= \min_{g \in \Hs^{\vec{\tau}}_{(j-1)} (\Ps)} L_{\bar{g}_\Ps (j)} (g, \Ps) + \tau_j (h) + \alpha
\end{align*}
where the first inequality follows from Equation~\eqref{eq:something3}. The third line follows from the definition of $\nu_j^2$. The last equality follows from Claim~\ref{clm:something4} and the fact that $\tau_j (h) = \epsilon_j (h) + \nu_j^1 (h) + \nu_j^2$. The proof is complete by the uniform convergence bound provided in the theorem statement and the fact that if two vector of group errors (in our case, one computed on the data set $S$ and another on the distribution $\Ps$) are close component-wise, then their sorted versions are also close component-wise (See Fact~\ref{fact:order}). Therefore, with probability at least $1-\delta$ over the random draws of the data set $S$, we have $\max_{h' \in \Hs} \vert \nu_j^1 (h') \vert \le \beta(\delta)$ and $\vert \nu_j^2 \vert \le \beta(\delta)$, and hence for all $j \le \ell$,
\begin{align*}
\Vert \tau \Vert_\infty &= \max_{1 \le j \le \ell} \left\{ \max_{h' \in \Hs} \tau_j (h') \right\} \\
&\le \max_{1 \le j \le \ell} \left\{ \max_{h' \in \Hs} \epsilon_j (h') \right\} + \max_{1 \le j \le \ell} \left\{ \max_{h' \in \Hs} \vert \nu_j^1 (h') \vert + \vert \nu_j^2 \vert \right\}\\
&\le \alpha + 2 \beta(\delta)
\end{align*}
%where the equality follows from Claim~\ref{clm:something2} and the first inequality follows from the definition of $\tau_j$.
\end{proof}

Same as before, we can now instantiate the above theorem in a classification setting where we have standard VC-type uniform convergence bound.

\begin{cor}[Generalization for Lexifairness: Classification Setting]
Suppose $\Hs$ is a class of binary classifiers with VC dimension $d_\Hs$ and let $L_z( p ) = \E_{h \sim p} \left[ L_z (h )\right]$ for any randomized model $p \in \Delta \Hs$ where $L_z (h) = \1 \left\{ h(x) \neq y \right\}$ is the zero-one loss. We have that for every $\Ps$, every data set $S \equiv \{ G_k \}_k$ of size $n$ sampled $i.i.d.$ from $\Ps$, if a model $p \in \Delta \Hs$ satisfies $(\ell, \alpha)$-lexicographic fairness with respect to $S$, then with probability at least $1-\delta$ it also satisfies $(\ell, 2 \alpha)$-lexicographic fairness with respect to $\Ps$ provided that
\[
\min_{1 \le k \le K} \left\vert G_k \right\vert = \Omega \left( \frac{ d_\Hs \log \left( n \right) + \log \left( K / \delta \right) }{\alpha^2} \right)
\]
\end{cor}